%% file: main.tex
\documentclass[10pt]{article}
\usepackage[utf8]{inputenc}
\usepackage[margin=.8in]{geometry}
\usepackage{rotating}
\usepackage{booktabs}
\usepackage{array}
\usepackage{amsmath, amsfonts}
\usepackage{color}
\usepackage{hyperref}
\usepackage{bbm}
\usepackage{algorithm}
\usepackage[noend]{algpseudocode}
\usepackage{subcaption}
\usepackage{graphbox}

\usepackage{mathtools}
\mathtoolsset{showonlyrefs}

\usepackage[
    backend=biber,
    style=ieee,
    citestyle=numeric-comp,
    sorting=none,
    giveninits=true,
    natbib,
    hyperref,
    maxbibnames=99,
    doi=false,isbn=false,url=false,eprint=false
]{biblatex}

\usepackage{amsthm}
\usepackage[normalem]{ulem}
\usepackage{soul}

\newtheorem{theorem}{Theorem} 
\newtheorem{prop}{Proposition}
\newtheorem{cor}{Corollary}[section]
\newtheorem{definition}{Definition}

\usepackage{chngcntr}
\usepackage{apptools}
\AtAppendix{\counterwithin{theorem}{section}}
\AtAppendix{\counterwithin{prop}{section}}

\usepackage{mathtools}

\usepackage[dvipsnames]{xcolor}

\def\Ical{\mathcal{I}}

\def\R{\mathbb{R}}
\def\E{\mathbb{E}}

\newcommand{\Y}{\mathcal{Y}}
\newcommand{\X}{\mathcal{X}}

\newcommand{\T}{\mathcal{T}}
\newcommand{\Tlhat}{\mathcal{T}_{\lhat}}

\newcommand{\Tlam}{\mathcal{T}_{\lambda}}

\newcommand{\lhat}{\hat{\lambda}}
\newcommand{\Lhat}{\widehat{\Lambda}}
\newcommand{\Rhat}{\widehat{R}}
\newcommand{\ind}[1]{\mathbbm{1}\left\{#1\right\}}

\newcommand{\triangleq}{\stackrel{\triangle}{=}}

\newcommand{\fdr}{\mathrm{FDR}}

\newcommand{\fdp}{\mathrm{FDP}}

\newcommand{\iou}{\mathrm{IOU}}
\renewcommand{\P}{\mathbb{P}}

\newcommand{\Var}{\mathrm{Var}}

\newcommand{\lb}{\left(}
\newcommand{\rb}{\right)}
\newcommand{\Z}{\mathbb{Z}}
\newcommand{\td}[1]{\tilde{#1}}
\newcommand{\eps}{\epsilon}
\newcommand{\cZ}{\mathcal{Z}}
\newcommand{\swap}{\mathrm{swap}}

\makeatletter
\def\blfootnote{\xdef\@thefnmark{}\@footnotetext}
\makeatother

\newcommand{\lihua}[1]{{\color{black} #1}}

\title{\vspace{-1.5cm} Learn then Test:\\Calibrating Predictive Algorithms to Achieve Risk Control}
\author{Anastasios N. Angelopoulos, Stephen Bates, Emmanuel J. Cand\`es, Michael I. Jordan, Lihua Lei}
\date{\today}

\begin{document}

\maketitle

\begin{abstract}
We introduce a framework for calibrating machine learning models so that their predictions satisfy explicit, finite-sample statistical guarantees. 
Our calibration algorithms work with any underlying model and (unknown) data-generating distribution and do not require model refitting.
The framework addresses, among other examples, false discovery rate control in multi-label classification, intersection-over-union control in instance segmentation, and the simultaneous control of the type-1 error of outlier detection and confidence set coverage in classification or regression.
Our main insight is to reframe the risk-control problem as multiple hypothesis testing, enabling techniques and mathematical arguments different from those in the previous literature.
We use the framework to provide new calibration methods for several core machine learning tasks, with detailed worked examples in computer vision and tabular medical data.
\end{abstract}

\section{Introduction}
\label{sec:introduction}
\begin{figure}[t]
    \vspace{-0.3cm}
    \centering
    \includegraphics[width=\textwidth]{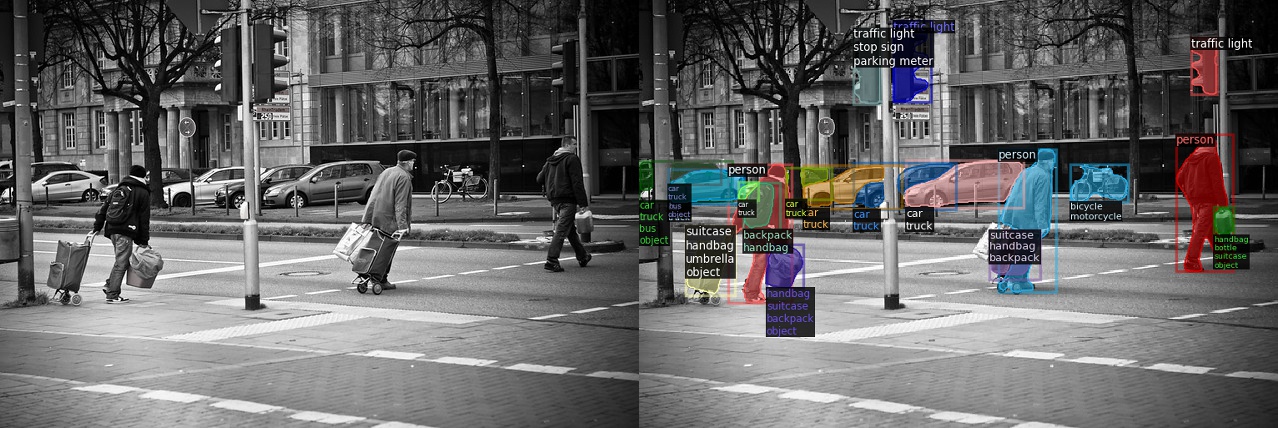}
    \caption{\textbf{Object detection with simultaneous distribution-free guarantees} on the expected intersection-over-union, recall, and coverage rate is possible with our methods; see Section~\ref{sec:detection} for details.}
    \label{fig:teaser}
\end{figure}

Deep neural networks and other complex machine learning models are increasingly used in real-world systems. 
Their use is motivated by the fact that (a) they often give the best prediction accuracy and (b) they extend to prediction tasks with structured outputs, such as image segmentation, predicting 3D volumes, sentence completion, and so on. 
However, the standard statistical toolbox does not apply to these models, so users seeking assurances about the reliability of predictions, such as bounds on the false discovery rate across multiple decisions, are left without recourse. 
For example, the Facebook object detector we used to produce Figure~\ref{fig:teaser} is currently deployed without any rigorous notion of uncertainty. 
Knowing when and how badly a model is likely to fail is critical for safe deployment; we urgently need new statistical tools for this task.

We develop a tool to endow complex prediction models with statistical guarantees. 
Our proposed \emph{\textbf{Learn then Test}} (LTT) framework and bestows finite-sample guarantees on any predictive model, without assumptions on the inner workings of the model or the true distribution underlying the dataset.
Our approach is modular in this sense, allowing it to keep pace with the ever-evolving ecosystem of neural network architectures, fitting schemes, and data types. 

In LTT, we begin with a \textbf{learned} model $\hat{f}$. 
We then post-process the model using calibration data in order to make our final predictions. 
The post-processing is controlled by a low-dimensional parameter $\lambda$. 
We \textbf{test} multiple values of the parameter using the calibration data in order to find settings that control a user-chosen statistical error rate. 
Thus, we link the problem of distribution-free predictive inference with the well-developed ideas in multiple hypothesis testing in order to ensure the reliability of machine learning models.

As an example, a multi-label classification model $\hat{f} : \X \to [0,1]^K$ generally outputs the estimated probability that each of $K$ classes is in an image. 
The parameter $\lambda$ is the binarization threshold that determines which classes we include in the prediction $\Tlam(x) = \{ k : \hat{f}(x)_k \geq \lambda \}$. 
Using fixed sequence testing, we can certify a choice of threshold $\lhat$ such that the resulting prediction $\Tlhat(X)$ has a false discovery rate (FDR) of no more than $10\%$ with the true set of classes in the image $Y$.
We apply our technique to this problem in Section~\ref{sec:multilabel}, and also to four other core problems in machine learning: selective classification, selective regression, out-of-distribution (OOD) detection, and instance segmentation.
The resulting guarantees are the first of their kind; all previous work in distribution-free uncertainty quantification, such as conformal prediction~\citep{vovk2005algorithmic} and risk-controlling prediction sets~\citep{bates2021distribution}, have required that $\lambda$ is one-dimensional and that the risk function be monotonic in $\lambda$. 
We do not make these assumptions, and can thus control many possibly non-monotone risks.
Code available at \textcolor{blue}{\url{https://github.com/aangelopoulos/ltt}}, allows reproduction of our experiments.

\subsection{Formalizing our goal}
\label{sec:notation-goals}

Let $(X_i, Y_i)_{i = 1,\dots,n}$ be the \emph{calibration set}, an independent and identically distributed (i.i.d.) set of variables, where the feature vectors $X_i$ take values in $\X$ and the responses $Y_i$ take values in $\Y$. 
We will use the calibration set to place guarantees on the predictions of a pretrained machine learning model $\hat{f}$ mapping from $\X$ to some space $\mathcal{Z}$.
In the tumor segmentation example above, $\hat{f}$ outputs an estimate of the probability that each of the $M \times N$ pixels comes from a tumor; thus, $\mathcal{Z}=[0,1]^{M \times N}$.
We post-process the raw model outputs in $\mathcal{Z}$ to generate predictions $\Tlam(x)$ indexed by a low-dimensional parameter $\lambda$.
Then, by carefully setting the parameter $\lhat$, we control a user-chosen error rate, regardless of the quality of $\hat{f}$ or the data distribution.

In the general framework, we allow the post-processing $\T_\lambda : \X \to \Y'$ to take values in any space $\Y'$.
We often choose $\Y'=\Y$, corresponding to predictions, or $\Y'=2^\Y$, corresponding to prediction sets. 
For such a function $\T_\lambda$, we define a \emph{risk} $R(\T_\lambda) \in \mathbb{R}$
that captures a problem-specific notion of the statistical error rate---like one minus the intersection-over-union (IOU) in our running example.
For convenience, we often use the shorthand $R(\lambda)$ to stand for $R\left(\T_{\lambda}(X)\right)$.
Our goal in this work is to train a function $\T_{\lhat}$ based on $\hat{f}$ and the calibration data in such a way that it achieves the following 
error-control property:
\begin{definition}[Risk-controlling prediction]
Let $\lhat \in \Lambda$ be a random variable. We say that $\T_{\lhat}$ is an \emph{$(\alpha,\delta)$-risk-controlling prediction (RCP)} if $\P(R(\T_{\lhat}) \le \alpha) \ge 1-\delta$. 
\label{def:rcp}
\end{definition}
\noindent The risk tolerance $\alpha$ and error level $\delta$ are chosen in advance by the user. 
The reader can think of 10\% as a representative value of $\delta$; the choice of $\alpha$ will vary depending on the risk function. 
In our work, $\hat{\lambda}$ will be a function of the calibration data, so the probability in the above definition will be over the randomness in the sampling of $(X_1,Y_1),\dots, (X_n,Y_n)$.
Note that some risks cannot be controlled at every level $\alpha$ for every data-generating distribution; in such cases, we abstain from returning an RCP.

\subsection{Related work}
Predictions with statistical guarantees have been heavily explored in the context of set-valued predictions.
This approach dates back at least to tolerance regions (sets that cover a pre-specified fraction of the population distribution) in the 1940s \cite{wilks1941, wilks1942, wald1943, tukey1947}. See \cite{krishnamoorthy2009statistical} for a review of this topic. Recently, tolerance regions have been used to form prediction sets with deep learning models \cite{Park2020PAC, park2021shift}.
In parallel, conformal prediction \cite{vovk1999machine, vovk2005algorithmic, angelopoulos2021gentle} has been developed as a way to produce prediction sets with finite-sample statistical guarantees. One convenient, widely-used form of conformal prediction, known as split conformal prediction \cite{papadopoulos2002inductive, lei2013conformal}, uses data splitting to generate prediction sets in a computationally efficient way; see also \cite{vovk2015cross, barber2019jackknife} for generalizations that re-use data for improved  statistical efficiency. Conformal prediction is a generic approach, and much recent work has focused on designing specific conformal procedures to have good performance according to additional desiderata such as small set sizes \cite{Sadinle2016LeastAS}, coverage that is approximately balanced across regions of feature space \cite{barber2019limits, romano2019conformalized, izbicki2020flexible, romano2020classification, cauchois2020knowing, guan2020conformal, angelopoulos2020sets}, and errors balanced across classes \cite{lei2014classification, Sadinle2016LeastAS, hechtlinger2018cautious, guan2019prediction}. 
Recent extensions also address topics such as distribution estimation \cite{vovk2019conformal}, causal inference~\cite{lei2020conformal}, survival analysis~\cite{candes2021conformalized}, differential privacy~\cite{angelopoulos2021private}, outlier detection~\cite{bates2021testing}, speeding up the test-time evaluation of complex models~\cite{fisch2020efficient, schuster2021consistent}, the few-shot setting~\cite{fisch2021few}, handling dependent data~\cite{chernozhukov18a, dunn2020distributionfree}, and handling of testing distribution shift~\cite{tibshirani2019conformal, cauchois2020robust, hu2020distributionfree, bates2021testing, gibbs2021adaptive, Vovk2021testing}.

Most closely related to the present work is the technique of Risk-Controlling Prediction Sets~\cite{bates2021distribution} which extends tolerance regions and conformal prediction to give prediction sets that control other notions of statistical error.
The present work goes beyond that work to consider prediction with risk control more generally, without restricting the scope to confidence sets. This is possible because we solve a key technical limitation of that earlier work---the restriction to monotonic risks---so that the techniques herein apply to any notion of statistical error. 
The present work moves in the direction of decision-making, which has been only lightly explored in the context of conformal prediction~\cite{vovk2018decision, vovk2019universally}, with the existing literature taking a very different approach.

\section{Risk Control in Prediction}
\label{sec:theory}
This section introduces our proposed method for controlling the risk of a prediction.
We begin by grounding our discussion with a preview of the LTT procedure.

\subsection{Overview of the Learn then Test procedure}
Recalling Definition~\ref{def:rcp}, our goal is to find a function $\Tlhat$ whose risk is less than some user-specified threshold $\alpha$. 
To do this, we search across the collection of functions $\{\T_\lambda\}_{\lambda \in \Lambda}$ and estimate their risk on the calibration data $(X_1,Y_1),\dots,(X_n,Y_n)$.
The output of the procedure will be a set of $\lambda$ values, $\Lhat \subseteq \Lambda$ which are all guaranteed to control the risk, abbreviated as $R(\lambda) = R(\Tlam)$.
The Learn then Test procedure is outlined below.

\begin{enumerate}
    \item For each $\lambda_j$ in a discrete set $\Lambda = \{\lambda_1, ..., \lambda_N\}$, associate the null hypothesis $\mathcal{H}_j : R(\lambda_j) > \alpha$.
    Notice that \emph{rejecting} $\mathcal{H}_j$ corresponds to selecting $\lambda_j$ as a point where the risk is controlled.
    
    \item For each null hypothesis, compute a finite-sample valid p-value using a concentration inequality. We discuss how to pick good p-values in Section~\ref{sec:pvals}.
    \begin{figure}[H]
        \centering
        \includegraphics[width=0.4\linewidth]{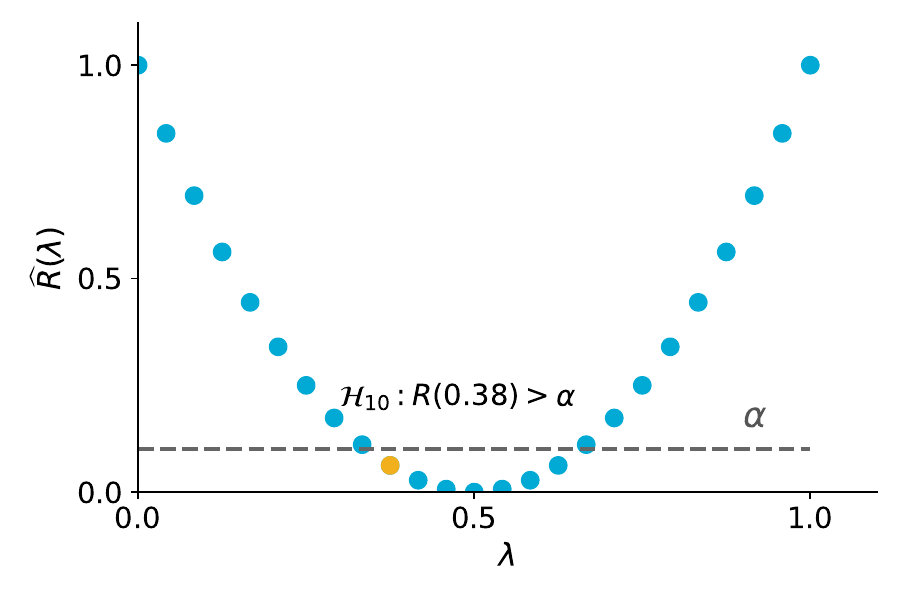}
        \includegraphics[width=0.4\linewidth]{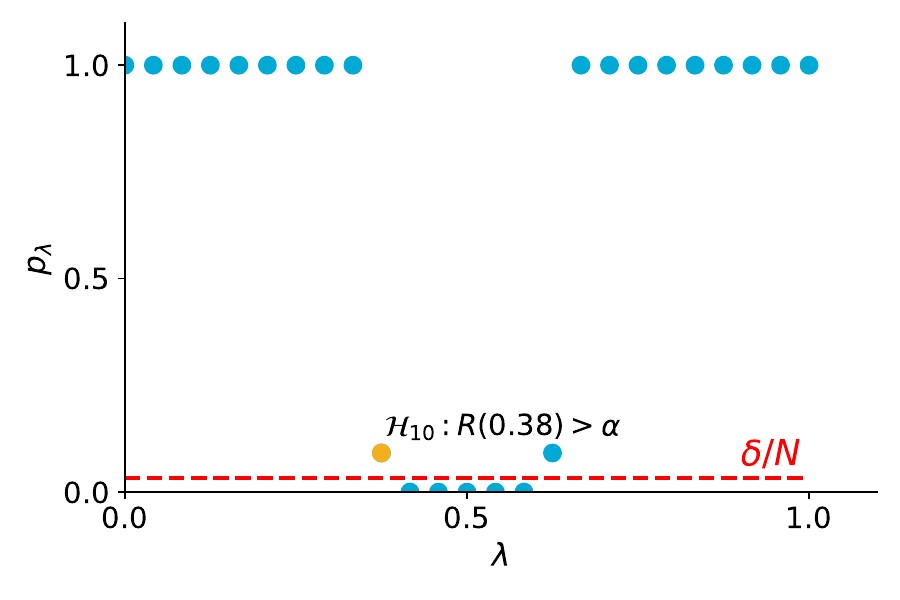}
    \end{figure}
    \item Return $\Lhat = \mathcal{A}\big(\{p_{j}\}_{j \in \{1,...,|\Lambda|\}}\big) \subset \Lambda$, where $\mathcal{A}$ is an algorithm that controls the family-wise error rate (FWER).
    For example, the Bonferroni correction yields $\Lhat=\big\{\lambda_j : p_{j} \leq \frac{\delta}{|\Lambda|}\big\}$. 
    In the image, this is all $\lambda$ values that fall below the red line.
    We formally define the FWER and show how to design good FWER-controlling procedures for machine learning settings in Section~\ref{sec:mht}.
\end{enumerate}

Except with probability $\delta$, each $\lhat \in \Lhat$ yields an RCP $\Tlhat$---this is our main result.
\begin{theorem}
    \label{thm:master-multiple-testing}
    Suppose $p_j$ has a distribution stochastically dominating the uniform distribution for all $j$ under $\mathcal{H}_j$.
    Let $\mathcal{A}$ be an FWER-controlling algorithm at level $\delta$. Then $\hat{\Lambda} = \mathcal{A}(p_1,\dots,p_N)$ satisfies the following: 
    \begin{equation}
    \label{eq:fwer_control}
        \P \left(\sup_{\lambda\in \Lhat}\{R(\lambda)\} \le \alpha \right) \geq 1-\delta,
    \end{equation}
    where the supremum over an empty set is defined as $-\infty$.
    Thus, selecting any $\lambda \in \Lhat$, $\T_{\lambda}$ is an $(\alpha,\delta)$-RCP.
\end{theorem}
All proofs are presented in Appendix~\ref{app:proofs}.
This result is straightforward but can be used to great effect: the user can use any FWER-controlling procedure to find $\Lhat$, and then may pick any $\lambda \in \Lhat$ as their chosen RCP (even in a data-driven way). 
For example, we can choose the one that maximizes another performance metric estimated using the same data without concern about double-dipping.
Most commonly, as in the FDR case, we would want to choose $\lhat = \min \Lhat$, which yields the most discoveries and hence the lowest false negative rate (FNR).

Theorem~\ref{thm:master-multiple-testing} reduces the problem of risk control into two subproblems: first, generating a p-value for each hypothesis, and second, combining the hypotheses to discover the least conservative prediction that controls the risk at level $\alpha$.
We will next formally develop solutions to each subproblem.

\subsection{Calculating valid p-values}
\label{sec:pvals}
The first step is to calculate a \emph{valid} p-value $p_j$ for each null hypothesis $\mathcal{H}_j$, i.e., one satisfying $u \in [0,1], \P(p_j \le u) \le u$ under $\mathcal{H}_j$.
The general idea is to calculate the empirical risk of $\T_{\lambda_j}$ for each $j$ then use an appropriate assumption-light concentration inequality to get the p-value for $\mathcal{H}_j : R(\lambda_j) > \alpha$. 
A small p-value will indicate disagreement with $\mathcal{H}_j$, implying the risk is controlled. 

Though our framework is more general, we consider the case where the risk function is the expectation of a \emph{loss} function $L$:
$R(\T)=\E[L(\T(X),Y)]$.
For example, the false discovery rate is the expectation of the false discovery proportion.
However, there are many other cases for which p-values are available, and our framework would automatically apply there as well.

Beginning with the bounded case where $L(\T(X),Y) \in [0,1]$, we apply the hybridized Hoeffding-Bentkus (HB) inequality from~\citet{bates2021distribution} which combines the results from~\citet{hoeffding1963} and~\citet{bentkus2004hoeffding}.
The HB inequality will use the empirical risk on the calibration set, $\Rhat_j = (1/n)\sum_{i=1}^{n}L(T_{\lambda_j}(X_i),Y_i) $, as the test statistic.
\begin{prop}[Hoeffding-Bentkus inequality p-values]
    \label{cor:hb-p-value}
    The following is a valid p-value for $\mathcal{H}_j$:
    \begin{equation}
        \label{eq:hb-p-value}
        p_j^{\rm HB} = \min\left( \exp\{-nh_1(\Rhat_j \wedge \alpha , \alpha)\}, e\P\big(\mathrm{Bin}(n,\alpha) \leq \left\lceil n\Rhat_j \right\rceil \big) \right),
    \end{equation}
    where $h_1(a,b)=a\log(a/b)+(1-a)\log((1-a)/(1-b))$.
\end{prop}
The Hoeffding-Bentkus inequality provides finite-sample statistical results with surprising empirical effectiveness, as we will see later.
In the unbounded case, the HB inequality no longer applies, but one can get asymptotically-valid p-values from the central limit theorem; see Appendix~\ref{app:clt-p-value} for details.

\subsection{Multiple Hypothesis Testing}
\label{sec:mht}
We next show how to combine p-values in order to form the rejection set $\Lhat$ using multiple hypothesis testing.
Our techniques will build on FWER control, which we formalize next.
Consider a list of null hypotheses, $\mathcal{H}_j$, $j=1,...,N$, with associated p-values $p_j$ that stochastically dominate the uniform distribution on $[0, 1]$ under the null.
Let the indices of the true nulls be $J_0 \subset \{1,\dots,N\}$ and those of the non-nulls be $J_1=\{1,...,N\}\setminus J_0$.
The goal of FWER control is to use the p-values to reject a subset of the $\mathcal{H}_j$ while limiting the probability of making any false rejections at a level $\delta$.
\begin{definition}[FWER-controlling algorithm]
\label{def:fwer-algo}
    An algorithm $\mathcal{A} : [0,1]^N \to 2^{\{1,...,N\}}$ is an \emph{FWER-controlling algorithm at level $\delta$} if $\P\left(\mathcal{A}\left(p_1,\dots,p_N\right) \subseteq J_1 \right) \geq 1-\delta$.
\end{definition}
\noindent Note that the p-values in the above definition may be dependent; we require that the algorithm works even in this case.
As a first step, we consider forming $\Lhat$ with the Bonferroni correction.
It is well-known that Bonferroni satisfies Definition~\ref{def:fwer-algo}.
\begin{prop}[Bonferroni controls FWER]
\label{prop:bonferroni}
Let $\mathcal{A}^\textnormal{(Bf)}(p_1,\dots,p_N) = \{\lambda_j : p_j \le \delta / |\Lambda|\}$.
Then, $\mathcal{A}^\textnormal{(Bf)}$ is an FWER-controlling algorithm.
\end{prop}
Bonferroni correction is simple, which makes it attractive.
However, for large $N$ the multiplicity correction degrades performance.
The remainder of the section is devoted to multiple testing methods that take advantage of problem structure to more intelligently search the hypothesis space, mitigating this issue.
This is possible because in our machine learning problems, adjacent p-values will be highly dependent for nearby $\lambda$, and non-nulls will generally cluster together.
Furthermore, we eventually only need one value $\lhat$ with reasonably good performance that guarantees the error control, while much of the multiple testing literature focuses on rejecting as many hypotheses as possible.

\subsubsection{Fixed sequence testing}

The multiple testing method herein is designed for settings where we know a-priori which hypotheses are more or less likely to control the risk.
For example, as in the case of the FDR, the risk function may be nearly monotone-decreasing in $\lambda$, making large $\lambda$ much more promising than small ones. 
In such settings we can dispense with multiplicity correction altogether if we sequentially test the hypotheses---e.g. from large $\lambda$ to small $\lambda$---as long as we stop upon the first acceptance.
This procedure, called \emph{fixed sequence testing}~\citep{bauer1991multiple}, controls the family-wise error rate.

More generally, we can initialize the fixed sequence test at several different points along the ordering, provided we adjust the significance level accordingly.
It is good to perform the \emph{multi-start} variant when nearby p-values are highly correlated, but the location of the minimum is not clear a-priori. 
This general procedure, stated formally in Algorithm~\ref{algo:bf_search}, is also guaranteed to control the type-1 error, as recorded in Proposition~\ref{prop:bf_search}.
\begin{algorithm}[H]
  \caption{Fixed sequence testing}
  \label{algo:bf_search}

  \begin{algorithmic}[1]
    \State \textbf{Input:} error level $\delta \in (0,1)$, parameter grid $\Lambda = \{\lambda_1,\dots,\lambda_N\}$, p-values $(p_1,\dots,p_N)$, initializations $\mathcal{J} \subset \{1,\dots,N\}$ (e.g., a coarse equi-spaced grid, with, say, 20 elements)
    \State $\Lhat \gets \emptyset$
    \For{$j \in \mathcal{J}$}
    \If{$\lambda_j \notin \Lhat$} \Comment{Avoid repeating values of $j$.}
        \While {$p_j \le \delta / |\mathcal{J}|$}
            \State $\Lhat\gets \Lhat\cup \{\lambda_j\}$
            \State $j \gets j + 1$
        \EndWhile
    \EndIf
    \EndFor
    \State \textbf{Return:} rejection set $\Lhat$
  \end{algorithmic}
\end{algorithm}
\begin{prop}[Fixed sequence testing controls FWER]
\label{prop:bf_search}
Algorithm~\ref{algo:bf_search} is an FWER-controlling algorithm; i.e., it satisfies Definition~\ref{def:fwer-algo}.
\end{prop}

Like Bonferroni, fixed sequence testing is simple to implement. 
In our experiments, we find that it offers large power improvements over Bonferroni, both in the \emph{multi-start} variant when $|\mathcal{J}|>1$ and in the standard setting where $|\mathcal{J}|=1$.
However, it involves an extra design-choice: selecting the ordering without looking at the calibration data.
Any ordering leads to valid error control, but some orderings will lead to better power than others.
When $\Lambda\subset \mathbb{R}$ and the loss function is nearly monotone, we can often take the natural ordering, as stated earlier.
When $\Lambda$ is higher dimensional, we discuss a general approach in Appendix~\ref{app:split-fixed-sequence} that learns an ordering using data splitting; we apply this new method to the instance segmentation detailed in Section~\ref{sec:detection}.
These choices matter to achieve tight risk. Still, we note that fixed sequence testing is yields exactly the desired FWER level:
\begin{prop}\label{prop:fixed_sequence_J=1}
Let $j^{*}$ be the index of the first null in the sequence. 
Then, for Algorithm~\ref{algo:bf_search} with $|\mathcal{J}| = 1$, $\mathrm{FWER} = \P(p_{j^{*}}\le \delta).$
As a result, if the null p-values are (asymptotically) uniform, the FWER is (asymptotically) $\delta$ as well.
\end{prop}
By contrast, Bonferroni typically yields a FWER much smaller than $\delta$.

\subsubsection{A general recipe for FWER control}
\label{sec:sgt}
Lastly, we introduce a general, more powerful framework for FWER control due to~\citet{brentz2009graphical}, which we call \emph{sequential graphical testing} (SGT). 
The SGT approach encodes information about the space of hypotheses $\Lambda$ via a directed graph, where the null hypotheses indexed by $\lambda \in \Lambda$ are the nodes, and the edges determine the way the error budget percolates through the graph. 
The procedure then sequentially tests the hypotheses indexed by $\lambda \in \Lambda$ at
iteratively updated
significance levels, while guaranteeing that the final $\Lhat$ controls the FWER.
The basic idea is simple: when a hypothesis is rejected, its error budget gets distributed among adjacent hypotheses in the graph, which allows them to be rejected more easily. The fixed sequence testing procedure (Algorithm \ref{algo:bf_search}) is a special case of SGT.

Formally, the SGT procedure is parameterized by a directed graph $\mathcal{G}$ comprising a node set $\Lambda$, and edge weights $g_{i, j} \in [0,1]$ for each pair $i,j \in \Lambda$  obeying $g_{i, i} = 0$ and $\sum_{j=1}^{n}g_{i, j}\le 1$. 
In addition, each node $i$ is allocated an initial error budget $\delta_i$ such that $\sum_i \delta_i = \delta$. 
From here, the algorithm tests each hypothesis $i \in \Lambda$ at level $\delta_i$ (i.e., checks if $p_i \le \delta_i$). 
If any $i$ is rejected, the procedure reallocates the error budget from node $i$ to the rest of the nodes according to the edge weights; see Algorithm~\ref{alg:graphical_fwer} for details.
This procedure, outlined in Algorithm~\ref{alg:graphical_fwer}, controls the family-wise error rate.

\begin{algorithm}
  \caption{Sequential graphical testing~\citep{brentz2009graphical}}
  \label{alg:graphical_fwer}

  \begin{algorithmic}[1]
    \State \textbf{Input:} error level $\delta \in (0,1)$, parameter grid $\Lambda = \{\lambda_1,\dots,\lambda_N\}$, p-values $(p_1,\dots,p_N)$, graph $\mathcal{G}$, initial error budget $\delta_i$ such that $\sum_i \delta_i = \delta$
    \State $\Lhat \gets \emptyset$
    \While{$\exists i : p_i \le \delta_i$}
    \State Choose any $i$ such that $p_i \le \delta_i$
    \State $\Lhat\gets \Lhat\cup \{\lambda_i\}$ \Comment{Reject hypothesis $i$}
    \State Update the error levels and the graph:
    \begin{align*}
        \delta_{j} &\gets 
        \begin{cases}
        \delta_{j} + \delta_i g_{i, {j}} & {\lambda_j} \in \Lambda \setminus \Lhat \\
        0 & \text{ otherwise }
        \end{cases} \text{ and } g_{k, j} \gets 
        \begin{cases}
        \frac{g_{k, j} + g_{k, i} g_{i, j}}{1 - g_{k, i}g_{i, k}} 
        & \lambda_k, \lambda_j \in \Lambda \setminus \Lhat, \ \ k \ne j \\
        0 & \text{ otherwise }
        \end{cases}
    \end{align*}
    \EndWhile
    \State \textbf{Return:} rejection set $\Lhat$
  \end{algorithmic}
\end{algorithm}
\begin{prop}[SGT controls FWER~\citep{brentz2009graphical}]
\label{prop:sgt}
Algorithm~\ref{alg:graphical_fwer} is an FWER-controlling algorithm; i.e., it satisfies Definition~\ref{def:fwer-algo}.
\end{prop}

The choices of the graph $\mathcal{G}$ and initial error budget $\{\delta_i\}_{i \in \Lambda}$ are critical for the power of the procedure.
The general principle is to concentrate the initial error budget on hypotheses likely to reject. 
If these promising hypotheses are indeed rejected, then the error budget should accrue to adjacent hypotheses, giving them a higher chance of rejected.
This allows the user to leverage structural information about the hypothesis space, as we do in Sections~\ref{sec:ood} and~\ref{sec:detection}. 

\subsection{Multiple risks and multi-dimensional $\lambda$}
\label{sec:multi-risk}
In this section, we show that our techniques apply to multiple risks and set constructions with multi-dimensional $\lambda$.

The formal modifications are straightforward. 
We allow $\Lambda$ with multiple dimensions and seek to control $m$ risks $R_1, ..., R_m$ at levels $\alpha_1, ..., \alpha_m$ simultaneously.
Thus, we define the null hypothesis
\begin{equation}
    \mathcal{H}_{j} : R_l(\lambda_j) > \alpha_l, \; \text{ for some } l \in{1,\dots,m}.
\end{equation}
To test this null hypothesis, we examine the finer null hypotheses, $\mathcal{H}_{j,l} : R_l(\lambda_j) > \alpha_l$.
Specifically, $\mathcal{H}_{j}$ holds if and only if there exists a $l \in 1, ..., m$ such that $\mathcal{H}_{j,l}$ holds,
which allows us to apply an FWER-controlling procedure to test $\mathcal{H}_j$, as we now summarize.
\begin{prop}
    \label{prop:multirisk-pvalue}
    Let $p_{j, l}$ be a p-value for $\mathcal{H}_{j,l}$, for each $l = 1, \dots, m$. Define $p_j := \underset{l}{\max}\: p_{j,l}$.
    Then, for all $j$ such that $\mathcal{H}_j$ holds and for all $u \in [0,1]$, we have $\P\left( p_j \le u \right) \leq u$.
\end{prop}
Having calculated valid p-values for each $\lambda_j$, we can now directly use the techniques from the previous section to select a set $\Lhat$ that controls the FWER. 

\subsection{An alternative approach: uniform concentration}
\label{sec:theory-uniform-explanatory}
An alternative approach is to use a \emph{uniform concentration bound} to control the risk.
If the upper bound lies above the true risk for \emph{all $\lambda$ at once} with high probability, we can produce $\Lhat$ by including all $\lambda$ where the bound falls below $\alpha$.
\begin{prop}[Uniform bounds give risk control]
\label{thm:master-uniform}
Let $R^{+}$ be a $1-\delta$ uniform upper confidence bound. i.e., $P(R(\T_\lambda)\le R^+(\T_\lambda)\text{ for any }\lambda\in \Lambda)\ge 1 - \delta.$
For any risk level $\alpha$, consider any $\hat{\lambda}$ such that $R^+\big(\T_{\hat{\lambda}}\big) \leq \alpha$.
Then $\Tlhat$ is an RCP.
\end{prop}
We develop a novel and carefully optimized concentration bound in Appendix~\ref{app:theory-uniform} and test its numerical behavior in Appendix~\ref{app:theory-numerics}.
The practical performance of uniform concentration is quite poor (see  Figures~\ref{fig:0_2_coco_histograms}-~\ref{fig:0_1_mse_meps}), which should be expected, as it solves a more difficult problem than is needed for risk control.

\section{Examples}
Next, we consider five large-scale machine learning tasks in computer vision and medicine, using LTT to provide new and useful finite-sample statistical guarantees.

\subsection{FDR Control for Multi-Label Classification}
\label{sec:multilabel}

\begin{figure}[ht]
    \centering
    \includegraphics[width=\linewidth]{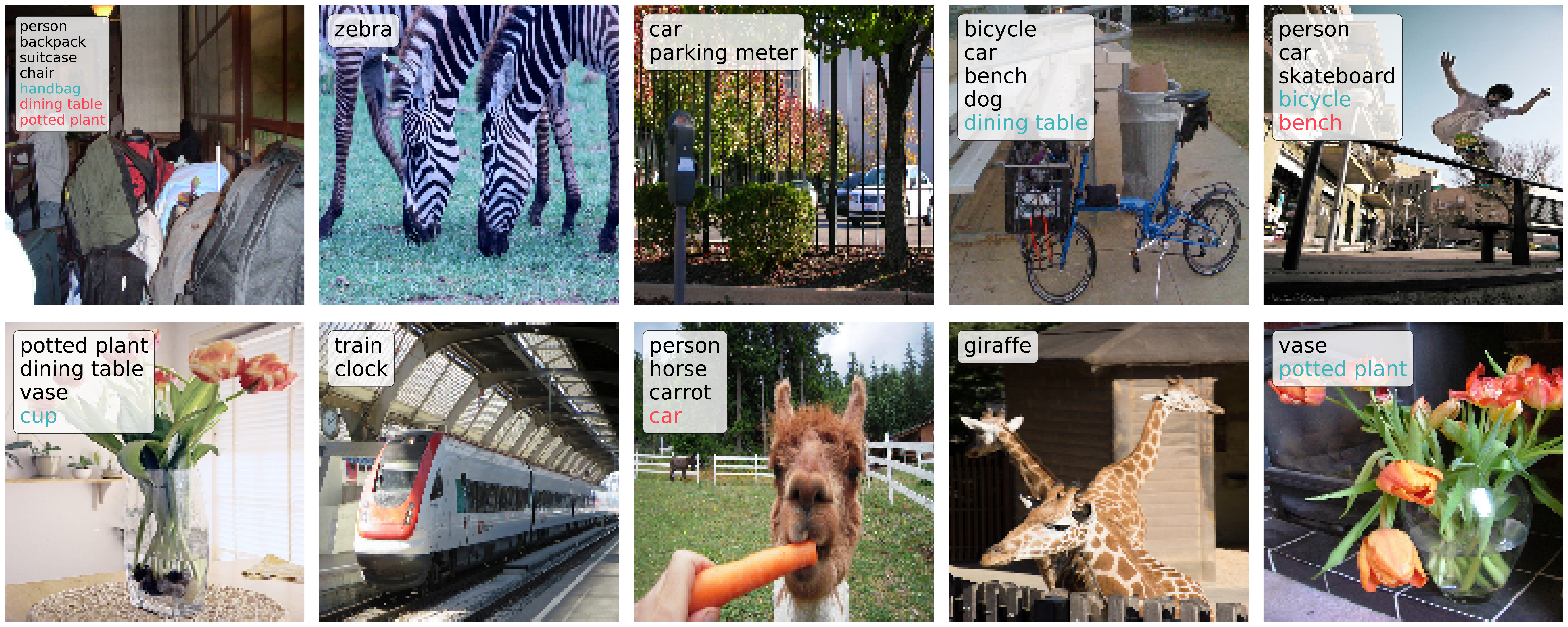}
    \caption{{\bf Multi-label prediction set examples on MS COCO} using fixed-sequence testing. Black classes are correct, blues are spurious, and reds are missed.}
    \label{fig:grid_coco_multiclass}
    \begin{minipage}{0.4\textwidth}
        \centering
        \includegraphics[width=\linewidth]{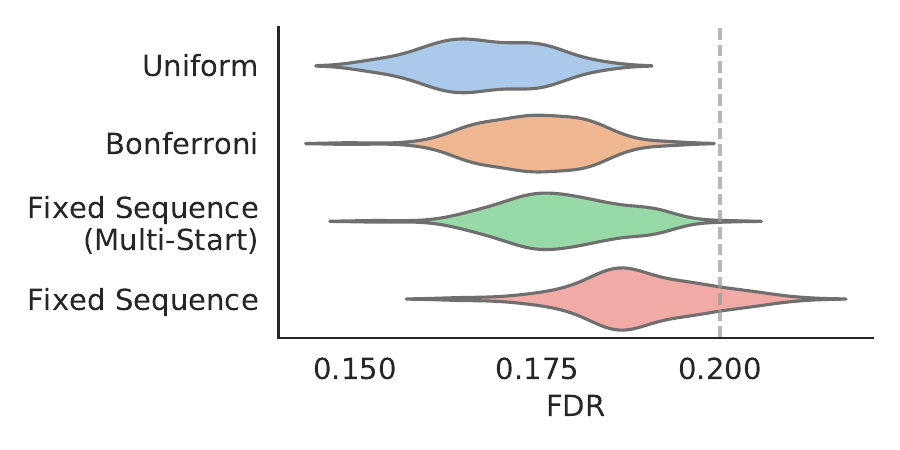}
    \end{minipage}
    \begin{minipage}{0.59\textwidth}
        \centering
        \input{figures/tables/0_2-coco-table}
    \end{minipage}
    \caption{{\bf Numerical results of our multi-label classification procedure.}
    The risk is plotted as a violin plot over 100 random splits of MS COCO, with parameters $\alpha=0.2$, shown as the gray dotted line, and $\delta=0.1$.
    The table shows quantiles of set sizes for each method (larger is better). 
    For details see Section~\ref{sec:multilabel}.}
    \label{fig:0_2_coco_histograms}
\end{figure}

We first consider multi-label classification, where each input $X$ (in our case, an image) may have multiple corresponding correct labels; i.e., the response $Y$ is a subset of $\{1,...,K\}$. 
We seek to return predictions that control the FDR at level $\alpha$,
\begin{equation}
\label{eq:multiclass-loss}
\mathrm{FDR}(\T) = 1 - \E\left[\frac{|Y \cap \T(X)|}{|\T(X)|}\right],
\end{equation}
where $Y$ is the ground truth label set and $\T$ is the prediction set function.
That is, we want to predict sets that contain no more than $\alpha$ proportion of false labels on average.
In this case, we take our prediction to be the following set of classes:
\begin{equation}
\label{eq:multiclass-sets}
    \T_\lambda(x) = \big\{z \in \{1,...,K\} :\hat{f}_z(x) \geq \lambda\big\},
\end{equation}
where $\hat{f} : \X \to [0,1]^K$ is a base classifier that outputs estimated probabilities that each class is present in the image.
The set-valued prediction $\T_{\lambda}(X)$ contains all sufficiently probable classes (as judged by our base model).
LTT then chooses $\lambda$ in such a way that guarantees FDR control at level $(\alpha, \delta)$.

We use $n=4000$ data points from the 80-class Microsoft Common Objects in Context (MS COCO) computer vision dataset~\citep{lin2014microsoft} and $1000$ for validation.
We plot visual and numerical results of several methods from Section~\ref{sec:theory} for $\alpha=0.2$, $\delta=0.1$, and $\Lambda=\{0,0.001, \ldots, 1\}$ in Figures~\ref{fig:grid_coco_multiclass} and ~\ref{fig:0_2_coco_histograms}.
See Appendix \ref{app:experiments} for more results and details.

\subsection{Selective Classification with Accuracy Control}
\label{sec:selective-classification}
\begin{figure}[b]
    \centering
    \includegraphics[width=\linewidth]{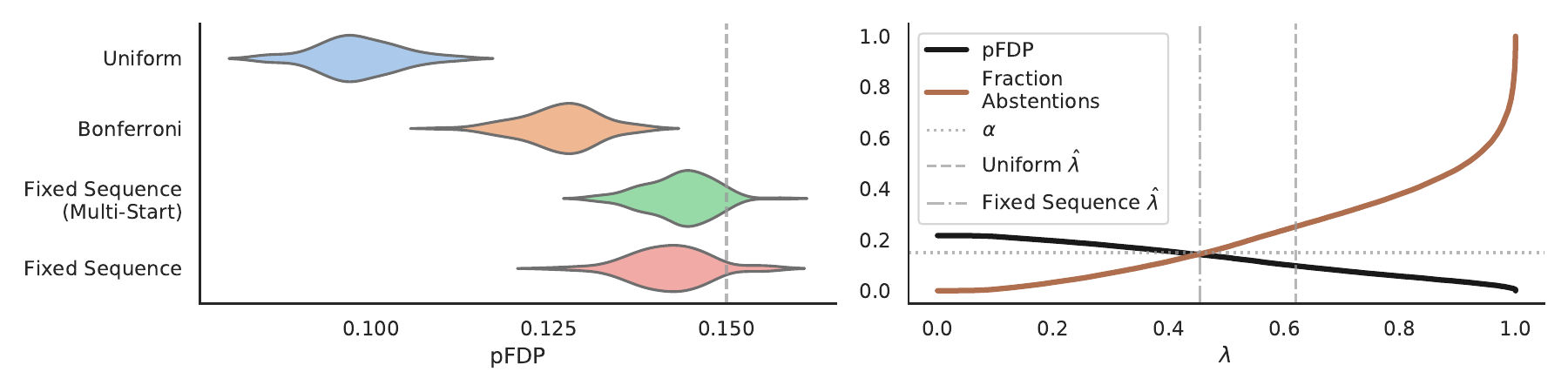}
    \caption{\textbf{Numerical results of selective classification on Imagenet.}
    The violins plot the selective error over 100 data splits at levels $\alpha=0.15$ and $\delta=0.1$. 
    The line plot shows the empirical risk and fraction of abstentions when sweeping across values of $\lambda$.}
    \label{fig:pfdp_0_15}
\end{figure}
We now consider the selective classification problem, sometimes called classification with a reject option or classification with abstention.
Here, $\mathcal{Y} = \{1,\dots,K\}$ for some $K$, and $\hat{f} : \X \to \Delta^K$ (the simplex on $K$ entries, since there can only be one true class).
For each test point, the classifier can either return the top-1 prediction or abstain by returning the empty set $\emptyset$, i.e., taking $y_{\mathrm{max}} = \arg\max_y\hat{f}_y(x)$, we predict
\begin{equation}
    \T_\lambda(x) = 
    \begin{cases}
         \big\{y_{\mathrm{max}}\big\} & \hat{f}_{y_{\mathrm{max}}}(x) > \lambda \\
          \emptyset & \text{else}.
    \end{cases}
\end{equation}
The risk is the error conditionally on predicting,
\begin{equation}
\label{eq:selective-error-def}
    \mathcal{E}(\T) := \P\left[ \T(X) \neq Y \bigg\lvert \ \T(X) \neq \emptyset \right].
\end{equation}
Due to the conditioning, we must calculate the empirical risk only among the subsample of data points that pass the threshold,
\begin{equation}
    \label{eq:selective-empirical-risk}
    \widehat{\mathcal{E}}(\T) = \frac{1}{n(\T)} \sum\limits_{i=1}^n \ind{\T(X_i) \neq Y_i \text{ and } |\T(X_i)| > 0} \text{, where } n(\T) = \sum\limits_{i=1}^n \ind{|\T(X_i)| > 0}.
\end{equation}
Because selective classification involves control of a binary loss, we use the exact binomial tail bound (the HB bound without the extra factor $e$) to calculate the p-value.

We now demonstrate selective classification on the Imagenet~\citep{deng2009imagenet} computer vision dataset using the model from~\citet{he2016deep}.
We report numerical results in Figure~\ref{fig:pfdp_0_15}, using 5K calibration points and 45K validation points with $\alpha=0.15$, $\delta=0.1$, and $\Lambda=\{0,0.001, \ldots, 1\}$. 
To avoid degeneracy, we clip all $\lambda$ with $n(\Tlam) < 25$.

\subsection{Selective Regression with MSE Control}
\label{sec:meps}
\begin{figure}[h]
    \centering
    \includegraphics[width=\linewidth]{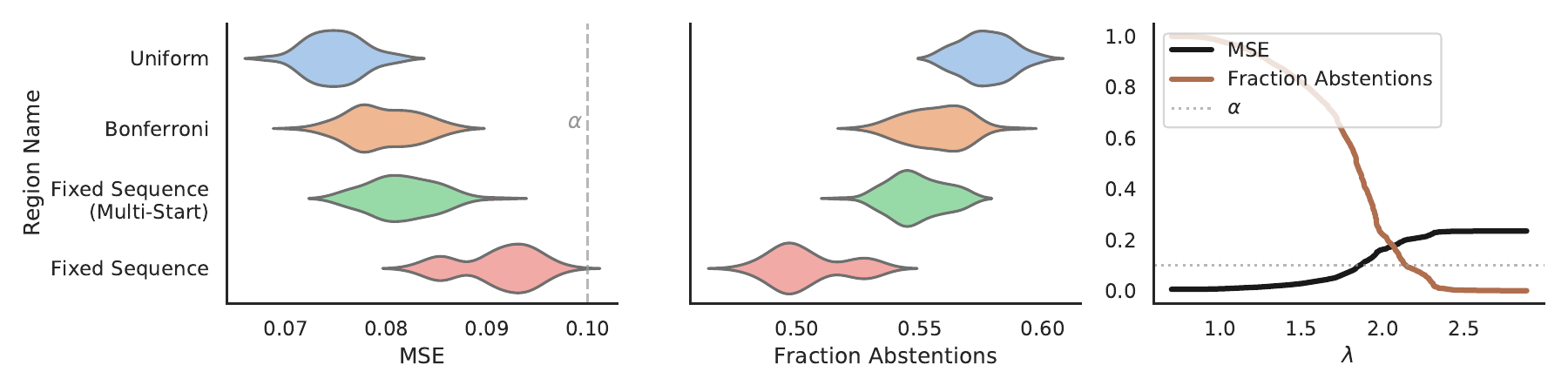}
    \caption{{\bf Numerical results of selective regression on the MEPS dataset.}
    The MSE is plotted as a violin plot over 100 random splits of the MEPS data, with parameters $\alpha=0.1$, shown as the gray dotted line, and $\delta=0.1$.
    The fraction of abstentions is plotted similarly.
    The line plot shows the tradeoff.
    For details, see Section~\ref{sec:meps}.}
    \label{fig:0_1_mse_meps}
\end{figure}
In the next example, we perform a \emph{selective regression} task, abstaining when the model is least confident to control the mean-squared error (MSE).

Formally, each input $X$ is a vector in $\X = \R^d$, where each of the $d$ dimensions signifies a covariate (a column index in the tabular dataset). 
The response $Y \in \R_{+}$ is in this case a medical expenditure.
The goal is to predict when we can be assured a low MSE and to abstain otherwise; that is, we control the conditional risk
\begin{equation}
    \label{eq:mse-risk}
    \mathrm{MSE}(\T) = \E\Big( \big(\T(X) - Y\big)^2 \Big\vert \T(X) \ne \emptyset \Big),
\end{equation}
where $\T$ takes values in $\mathbb{R}$ or the value $\emptyset$ signifying abstention.
We can obtain finite-sample valid p-values by calculating the empirical MSE on the subset of data points that pass the selection, as in Section~\ref{sec:selective-classification}.

For the set construction, we will use a model $\hat{f} : \X \to \R_{+}$ estimating the conditional mean and a second model $\hat{r} : \X \to \R_{+}$ estimating the magnitude of the residual of $\hat{f}$.
We abstain when the estimated residual magnitude is large,
\begin{equation}
    \T_\lambda(x) = 
    \begin{cases}
           \hat{f}(x) & \hat{r}(x) \leq \lambda \\
          \emptyset & \text{else}
    \end{cases}.
\end{equation}
We demonstrate results on the 2019 Medical Expenditure Prediction Survey (MEPS) dataset~\citep{cohen2009medical} in Figure~\ref{fig:0_1_mse_meps} for $\alpha=\delta=0.1$.
The MEPS dataset contains 15784 data points consisting of covariates about households paired with their yearly medical expenditures; we seek to predict the latter from the former.
We used $7892$ data points to train $\hat{f}$ and $\hat{r}$, $n=3946$ calibration data points, and $3946$ validation data points.
For $\hat{f}$ and $\hat{r}$ we use gradient boosting trees~\citep{friedman2001greedy}.

\subsection{Prediction Sets with OOD Detection}
\label{sec:ood}
Next, we seek to make predictions that either declare an input to be out-of-distribution (OOD) or return a prediction set with a coverage guarantee.
In this setting, we want to detect as many out-of-distribution inputs as possible while only falsely rejecting some pre-specified fraction (e.g., 1\%) of in-distribution inputs.
Furthermore, within the subset of inputs deemed in-distribution, we ask for prediction sets with coverage $1 - \alpha_2$ (e.g., 99\%).
This results in a two-dimensional risk function with coordinates
\begin{equation}
    R_1(\T) = \P\left( |\T(X)| = 0\right)\;\;\;\text{   and       }\;\;\;R_2(\T) = \P\left( Y \notin \T(X) \mid |\T(X)| > 0 \right).
\end{equation}
We apply LTT as in~Section~\ref{sec:multi-risk} such that both $R_1(\T) \le \alpha_1$ and $R_2(\T) \le \alpha_2$ hold with probability at least $1-\delta$.

We will abstain by outputing $\emptyset$ when the example is likely to be OOD based on an OOD detector $\mathrm{OOD}(x)$ and otherwise output prediction sets.
If $\mathrm{OOD}(x) \geq \lambda_1$, we output $\emptyset$.
Otherwise, as in conformal prediction, we use a score function $s(x,y)$ that controls which classes are in the prediction set, and threshold it at $\lambda_2$.
Concretely, the set construction procedure is:
\begin{equation}
    \label{eq:ood-conformal-nested-sets}
    \T_\lambda(X) =  \begin{cases} 
                      \emptyset, & \mathrm{OOD}(X) \geq \lambda_1, \\
                      \{ y : s(X,y) \leq \lambda_2 \} \cup \underset{y}{\arg \min} \; s(X,y), & \mathrm{OOD}(X) < \lambda_1.
                   \end{cases}
\end{equation}
The output of LTT will be all pairs of $(\lambda_1, \lambda_2)$ that simultaneously control both the type-1 error of OOD detection, and the coverage conditional on prediction.

We design a specialized multiple-testing method to take advantage of the coordinate-wise monotonicity of the risk.
Notice that the risks are coordinate-wise monotonic, meaning if $\lambda_2$ is fixed, $\lambda_1^{(1)} \le \lambda_1^{(2)} \implies R_1\big((\lambda_1^{(1)},\lambda_2)\big) \leq R_1\big((\lambda_1^{(2)},\lambda_2)\big)$
and vice-versa when $\lambda_1$ is fixed.
Furthermore, for a fixed $\lambda_1$, we are most interested in the smallest $\lambda_2$, and vice versa; this essentially defines a Pareto frontier between $R_1$ and $R_2$.
We will use SGT as described in~\ref{sec:sgt} to efficiently identify values of $(\lambda_1, \lambda_2)$ that control the risk in order to try to maximize power.
We describe these testing methods, dubbed \emph{2D fixed sequence testing}, \emph{fallback SGT}, and \emph{hamming-equalized SGT}, in Appendix~\ref{app:hamming-sgt}.

\begin{figure}[t]
    \centering
    \includegraphics[width=\linewidth]{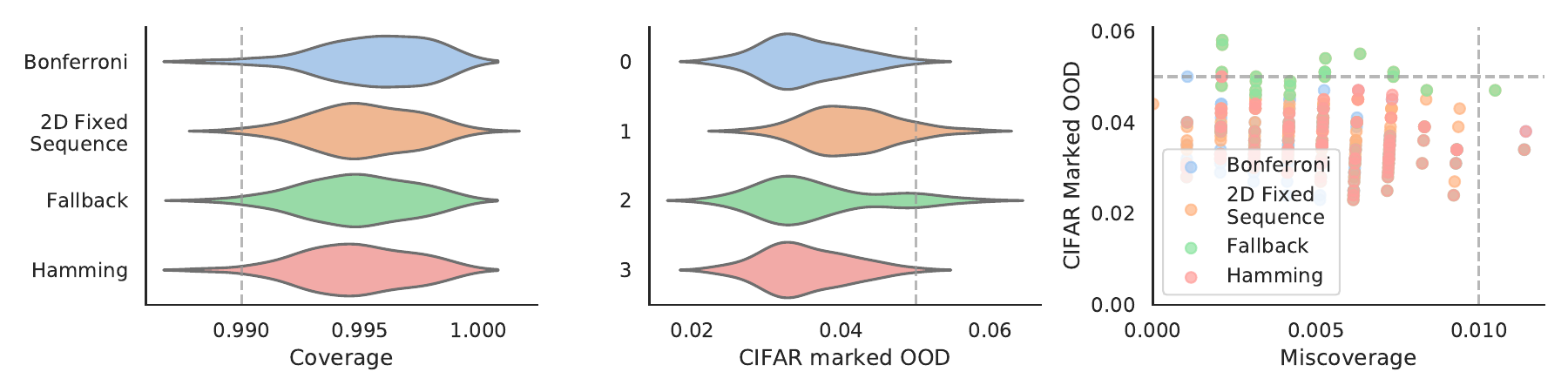}
    \caption{{\bf Numerical performance of methods for simultaneous OOD type-1 error and coverage control on CIFAR-10} with $\alpha_1=0.05$, $\alpha_2=0.01$, and $\delta=0.1$.
    The violins quantify coverage, type-1 error, and the power of the OOD procedure against Imagenet images over 1000 random data splits.
    The gray dotted lines show $\alpha_1$ and $\alpha_2$.}
    \label{fig:ood_bonferroni_rescaled}
\end{figure}

We evaluate these methods on the 10-class CIFAR-10 dataset~\citep{krizhevsky2009learning}. 
We take a DenseNet classifier $\hat{f}$ as our base model. 
As our OOD score, we use the ODIN method~\citep{liang2017enhancing}; see Appendix~\ref{app:odin-rescaling} for details. 
As our classification score function $s(x,y)$, we use the conformal score from the Adaptive Prediction Sets procedure~\citep{romano2020classification}.
Formally, the score function is
\begin{equation}
    \label{eq:aps-score}
    s(x,y) = \sum\limits_{j=1}^{k-1} \hat{f}_{\pi_j}(x) + U\hat{f}_{\pi_k}(x) \text{ where } y=\pi_k, \: U \sim \mathrm{Unif}(0,1),
\end{equation}
and $\pi$ is the permutation of $\{1,...,K\}$ that sorts $\hat{f}(X)$ from most to least likely.

The numerical performance of our method is reported in Figure~\ref{fig:ood_bonferroni_rescaled}. 
Here, we used $\Lambda = \{0, \frac{1}{1000}, \frac{2}{1000},...,1\}^2$ with 8000 calibration points and 2000 validation points. 
As expected, we see that we control both risks at the desired level.
To evaluate the power of the OOD detection substep, we also tested it on 10,000 downsampled images from Imagenet; it correctly identifies $>99\%$ of Imagenet images as OOD.
Thus, our strategy leads to reasonable results which are not conservative in coverage or OOD Type-1 error and easily distinguish against Imagenet.
Indeed, the SGT procedure violates the desired risks $10\%$ of the time when we set $\delta=10\%$; the procedure hits exactly the target level.

\section{Example: Instance Segmentation with mIOU, Coverage, and Recall Guarantees}
\label{sec:detection}
We finish with our flagship example: object detection.
We will focus on a variant of object detection called \emph{instance segmentation}, in which we are given an image and asked to (1) identify all distinct objects within the image, (2) segment them from their background, and (3) classify them.
See Figure~\ref{fig:detection-examples} for several examples of our procedure. 
These three goals can be formally encoded by evaluating a detector's recall, the \emph{Intersection-over-Union} (IOU) of the segmentation masks with the ground truth mask, and the misclassification rate, respectively.
Traditionally, these measures are combined into a single metric called \emph{average precision} to heuristically evaluate the performance of a detector~\cite{he2017mask}.
Although the average precision can also be handled by our methods, here we will take the stronger stance of \emph{controlling} all three error rates at the same time.

In more detail, we will use \texttt{recall}, \texttt{IOU}, and \texttt{coverage} as 
the error rates corresponding to the three subtasks of detection. (These error rates will be formally defined soon.) We set up our experiments such that the detector has three final tuning parameters, $\lambda_1, \lambda_2$ and $\lambda_3$, that each controls the detector's performance on one of the three tasks.
In particular, $\lambda_1$ tunes the number of objects that are selected, $\lambda_2$ tunes the size of the bounding regions, and $\lambda_3$ tunes the certainty level for classification. The setup can be summarized as follows: 
\begin{center}
\begin{tabular}{ l c c }
\textbf{Goal} & \textbf{Quantity to Control} & \textbf{Parameter} \\
 (1) Locate distinct objects & \texttt{recall} & $\lambda_1$ \\ 
 (2) Find the precise set of pixels for each object & \texttt{IOU} & $\lambda_2$ \\  
 (3) Assign the right class to each object & \texttt{coverage} & $\lambda_3$ 
\end{tabular}
\end{center}
In truth, this is a simplification, in that the parameters $\lambda_1, \lambda_2$ and $\lambda_3$ are not entirely disentangled: changing any one parameter can change all three error rates. Nonetheless, $\lambda_1$ primarily corresponds to the \texttt{recall} error rate, and so on. This will be made precise in the next section.
Lastly, before launching into the formal details, we note that state-of-the-art object detectors like \texttt{detectron2}~\cite{wu2019detectron2} have many substeps, and for the sake of brevity we will only explain the relevant ones.

\subsection{Formal specification of object detection}

\begin{figure}
    \centering
    \includegraphics{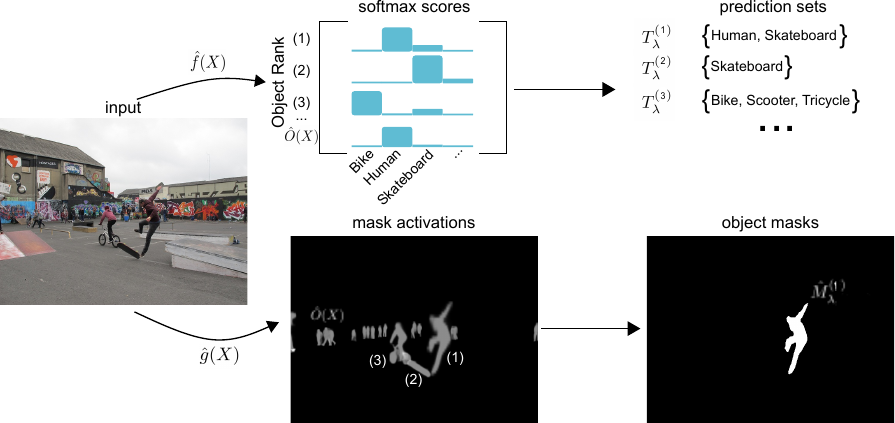}
    \caption{\textbf{Our detection pipeline} annotated with the formal mathematical notation to aid the reader.}
    \label{fig:my_label}
\end{figure}

Now we begin a formal treatment of the instance segmentation problem.
Our inputs are images $X_i \in \R^{H \times W \times D}$, $i=1,...,n$ where $H$, $W$, and $D$ are the height and width and channel depth respectively (in practice each image is a different size, but we ignore this for notational convenience).
Along with each input, we receive a set of tuples containing a mask and a class for each of the $O(X_i)$ number of objects in image $X_i$. 
The response $M_i^{(j)} \in \{0,1\}^{H \times W}$, $ j=1,...,O(X_i)$ represents a binary segmentation mask for the $j$th object in image $i$ and the response $C_i^{(j)} \in \{1,...,K\}$ represents the class of that object.
Thus, for a given $X_i$, the response $Y_i$ is the sequence of masks and classes
\begin{equation}
    Y_i = \left\{ \left(M_i^{(j)}, C_i^{(j)}\right)\right\}_{j=1}^{O(X_i)}.
\end{equation}
For our purposes, an object detector comprises three functions: $\hat{O}(X_i)$, which gives the number of predicted objects, $\hat{f}(X_i) \in [0,1]^{\hat{O}(X_i) \times K}$ \lihua{(with elements $\hat{f}_{k}^{(j)}(X_i)$ where $j\in \{1, \ldots, \hat{O}(X_i)\}$ and $k \in \{1, \ldots, K\}$)}, a set of estimated class probabilities for each of the  predicted objects, and $\hat{g}(X_i) \in [0,1]^{\hat{O}(X_i) \times H \times W}$ (\lihua{with elements $\hat{g}^{(j)}_{h, w}$ where $j\in \{1, \ldots, \hat{O}(X_i)\}$, $h \in \{1, \ldots, H\}$, and $w\in \{1, \ldots, W\}$}), representing the probability that a pixel comes from each predicted object.
In our eventual pipeline, the job of $\hat{f}$ will be to classify each predicted object, and the job of $\hat{g}$ will be to form binary segmentation masks for each.
It will soon become relevant that the number of predicted objects is not equal to the number of ground truth objects, i.e., $\hat{O}(X_i) \neq O(X_i)$.
Now, we must tackle the task of turning the raw outputs $\hat{f}(X)$ and $\hat{g}(X)$ into binary masks and classes, while filtering out those predictions deemed unreliable.

First, we select only the most confident detections for eventual display to the user.
In words, we will only output objects with a sufficiently high softmax score, since lower scores suggest that the object is more likely to be spurious.
We will call the indexes of those objects
\begin{equation}
    \tilde{J}_{\lambda}(X) = \{j : \lihua{\max_{k} \hat{f}_{k}^{(j)}(X)} \geq 1-\lambda_1 \} \subseteq \{1,\dots, \hat{O}(X) \}.
\end{equation}
We will use $\tilde{J}_{\lambda}(X)$ to ignore objects with top scores lower than the threshold $1-\lambda_1$ in later risk computations.

Second, we explain the construction of binary masks.
We set the $(h,w)$ coordinate of the estimated masks for the $j$th object as
\begin{equation}
    \hat{M}^{(j)}_{\lambda, (h,w)}(X) = \ind{\lihua{\hat{g}^{(j)}_{h,w}(X)} \geq  1 - \lambda_2 
    \text{, } j \in \tilde{J}_{\lambda}(X)}.
\end{equation}
Let us reflect on the construction of the mask.
If the pixel truly comes from the object, $\lihua{\hat{g}^{(j)}_{h,w}(X)}$ should be large, which motivates us to threshold these values at $1-\lambda_2$.
Also note that if the model is not sufficiently confident about the object's class, then $j \notin \tilde{J}_{\lambda}(X)$, causing the procedure to give up and output a mask of all zeros. 

Third, we discuss the formation of the prediction sets for the class of the object.
We set 
\begin{equation}
    T^{(j)}_{\lambda}(X) =  \begin{cases} 
                              \emptyset & j \notin \tilde{J}_{\lambda}(X) \\
                              \left\{ \lihua{k} : s^{(j)}(X,\lihua{k}) \leq \lambda_3 \right\} \cup \left\{ \lihua{\underset{k}{\arg \max}\,      \hat{f}_{k}^{(j)}(X)} \right\}  & j \in \tilde{J}_{\lambda}(X),
                           \end{cases}
\end{equation}
where $s^{(j)}(X,\lihua{k})$ is the score from~\eqref{eq:aps-score} of class $k$ on predicted object $j$ from image $X$, which is lower for more likely classes (see \eqref{eq:aps-score}).
Holding $\lambda_1$ fixed, the prediction set grows when $\lambda_3$ grows.
Also, as in the example of the mask construction, if the detector is not sufficiently confident about the top class, the procedure gives up and outputs the null set.

Having described both ingredients of our final output, we can now bring them together.
Given an input image $X$, our final prediction is a sequence of $\hat{O}(X)$ masks and prediction sets for the object class,
\begin{align}
    \hat{Y}_{\lambda} & = \left\{ \left(\hat{M}_{\lambda}^{(j)}(X), \hat{T}_{\lambda}^{(j)}(X)\right)\right\}_{j=1}^{\hat{O}(X)}.
\end{align}

Now we will formally define the risk functions. 
Before doing so, we begin by defining the \texttt{IOU} for two binary masks:
\begin{equation}
    \iou(M,\hat{M}) = \frac{\#\left\{ (h,w) : M_{(h,w)} = 1 \text{ and } \hat{M}_{(h,w)} = 1 \right\}}{\#\left\{ (h,w) : M_{(h,w)} = 1 \text{ or } \hat{M}_{(h,w)} = 1 \right\}}.
\end{equation}

\begin{figure}
    \centering
    \includegraphics[width=\textwidth]{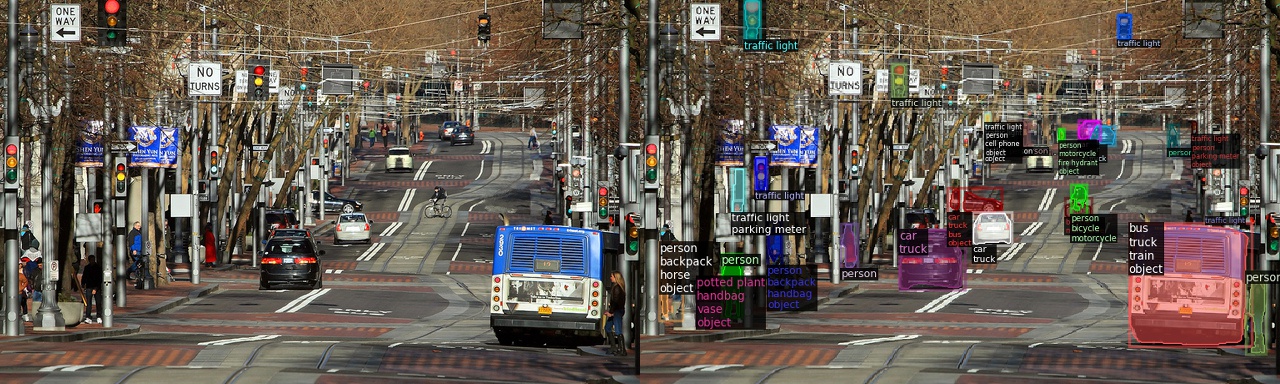}
    \includegraphics[width=1.0\textwidth]{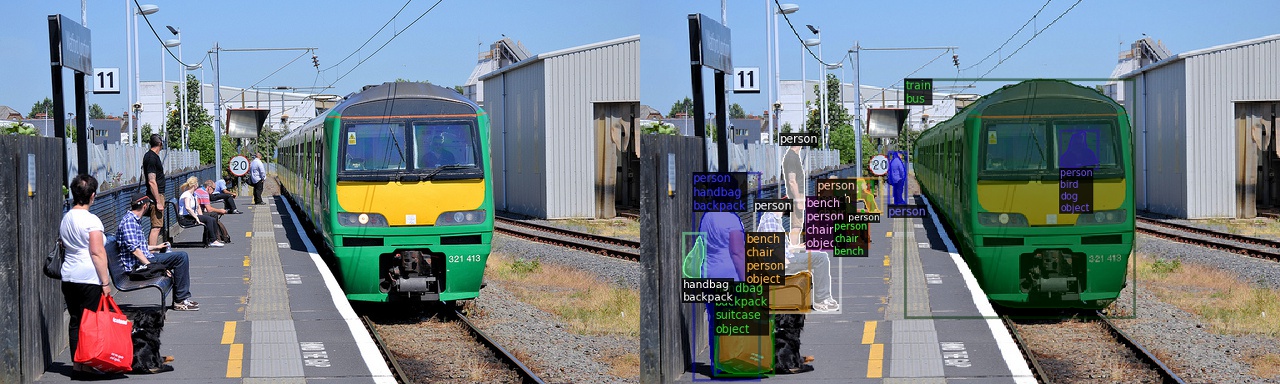}
    \includegraphics[width=\textwidth]{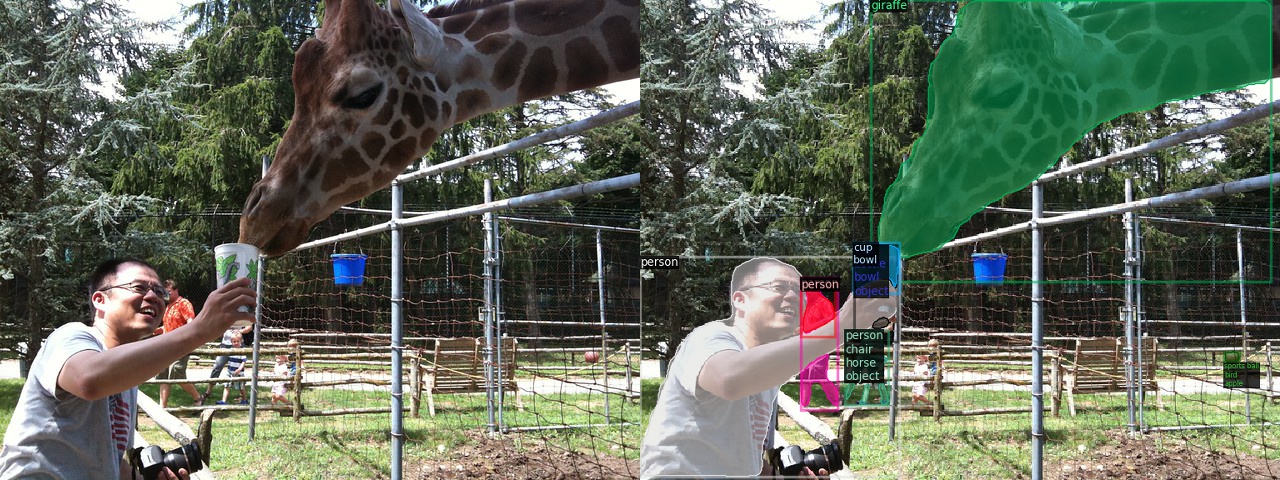}
    \caption{\textbf{Instance segmentation examples on MS COCO.} The input images are on the left, and the output segmentations are on the right.
    The segmentation masks are shown in random colors over their respective objects, and the prediction sets are included as color-coded lists of classes adjacent to the bounding box for each object.
    We produced these images by running the fixed-testing procedure with $\alpha_1=0.5$, $\alpha_2=0.5$, $\alpha_3=0.25$, and $\delta=0.1$.}
    \label{fig:detection-examples}
\end{figure}
The \texttt{IOU} is an FDR-like quantity for segmentation masks.
Importantly, the \texttt{IOU} function gets used to determine the correspondence between the predicted objects and the ground truth objects---recall that we do not know which predicted objects, if any, correspond to each ground truth object.
To match the predictions with the ground truth, we set $\tilde{\pi}$ and $\tilde{\pi}'$ to be the permutation maximizing
\begin{equation}
    \sum\limits_{j=1}^{\min\left(O(X),\hat{O}(X)\right)} \iou(M^{\tilde{\pi}(j)},\hat{M}_{\lambda}^{\tilde{\pi}'(j)}(X))\ind{j \in \tilde{J}_{\lambda}(X)}.
\end{equation}
From now on, assume without loss of generality that $\tilde{\pi}(j) = \tilde{\pi}'(j) = j$.
Finally, we define our three risks,
beginning with the negative of the \texttt{recall},
\begin{equation}
    R_1 = \E\left[1- \frac{1}{O(X)}\sum\limits_{j=1}^{\min\left(O(X),\hat{O}(X)\right)} \ind{C^{(j)} = \underset{k}{\arg \max} \hat{f}_{\lihua{k}}^{(j)}(X), j \in \tilde{J}_{\lambda}(X) } \right],
\end{equation}
continuing with the negative of the average \texttt{IOU},
\begin{equation}
    R_2 = \E\left[1-\frac{1}{|\tilde{J}_{\lambda}(X)|}\sum\limits_{j=1}^{\min\left(O(X),\hat{O}(X)\right)} \iou(M^{(j)}, \hat{M}_{\lambda}^{(j)}(X)) \ind{j \in \tilde{J}_{\lambda}(X)} \right],
\end{equation}
and ending with the negative of the image-wise average \texttt{coverage},
\begin{equation}
    R_3 = \E\left[1-\frac{1}{|\tilde{J}_{\lambda}(X)|}\sum\limits_{j=1}^{\min\left(O(X),\hat{O}(X)\right)} \ind{C^{(j)} \in \hat{T}_{\lambda}^{(j)}(X), j \in \tilde{J}_{\lambda}(X)}\right].
\end{equation}
The reader should notice that in $R_2$ and $R_3$, the term $|\tilde{J}_{\lambda}(X)|$ penalizes spurious detections.

\begin{figure}
    \centering
    \includegraphics[width=1.0\textwidth]{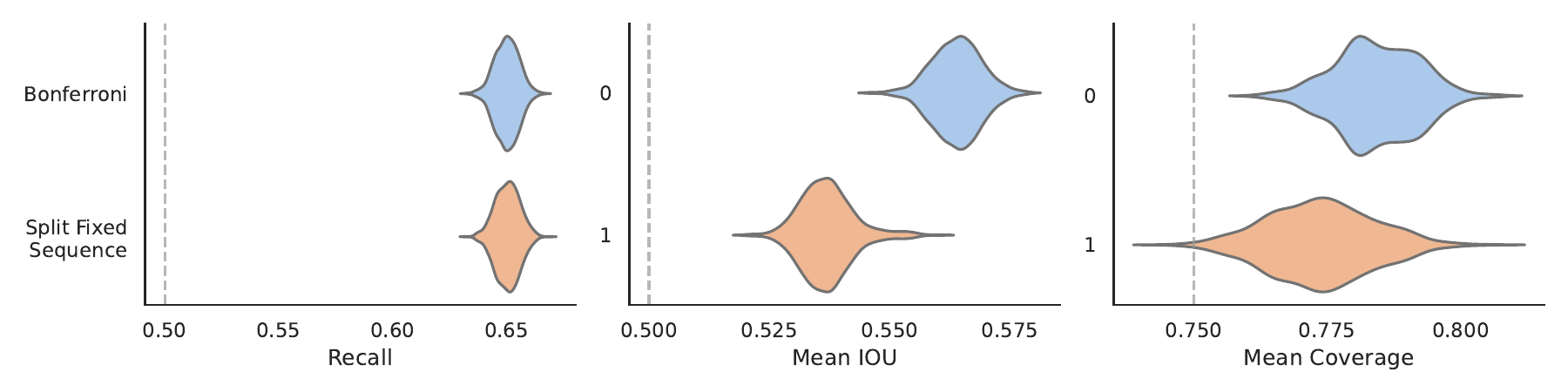}
    \caption{\textbf{Numerical results for our instance segmentation algorithm.}
    In order from left to right, we plot the mean \texttt{coverage}, mean \texttt{IOU}, and \texttt{recall} over 1000 random splits of the MS-COCO dataset. The gray dotted lines represent the coordinates of $\alpha$.
    We chose $\alpha_1=0.5$, $\alpha_2=0.5$, $\alpha_3=0.25$, and $\delta=0.1$.
    We define the split fixed sequence testing procedure in Appendix~\ref{app:split-fixed-sequence}.
    Here, we have chosen $\hat{\lambda}$ in such a way as to make coverage as tight as possible for both procedures. This leads to the conservativeness of \texttt{recall} and \texttt{IOU}; see the discussion at the end of Section~\ref{sec:detection}.
    }
    \label{fig:detection-histograms}
\end{figure}
As with our FWER-controlling procedure, we developed a new method called split fixed sequence testing, described in detail in Appendix~\ref{app:split-fixed-sequence}.
The basic inspiration for split fixed sequence testing is that when it is unclear how to construct the underlying graph for SGT, we can learn it from an extra split of data.
As a final remark on the theory, once we select a set $\widehat{\Lambda}$ with a FWER-controlling procedure as in Proposition~\ref{prop:multirisk-pvalue}, we still have to select the specific element $\lhat \in \widehat{\Lambda}$ to report our results.
Because we use FWER control to pick $\widehat{\Lambda}$, we can optimize over the parameters however we wish.
Accordingly, we follow the following procedure (note that all coordinates of $\lambda$ depend on each other, and must be chosen jointly):
\begin{align}
    \lhat_3 & = \underset{\lambda \in \Lhat}{\min} \left\{\lambda_3 : \exists \lambda_1 \in \Lhat,  1-\lambda_1 < \lambda_3\right\}\\
    \lhat_1 & = \underset{\lambda \in \Lhat}{\max} \left\{\lambda_1 : \lambda_3 = \lhat_3\right\} \\
    \lhat_2 & = \max_{\lihua{\lambda\in\Lhat}} \left\{ \lambda_2 : \left[\lhat_1, \lambda_2, \lhat_3\right] \in \underset{\lambda \in \Lhat}{\arg \min} \; \hat{R}_2(\lambda) \right\} ,
\end{align}
which optimizes over $\Lhat$ to get meaningful prediction sets, many detections, and a high IOU.

We again used the MS-COCO dataset for experiments, this time using Facebook AI Research's \texttt{detectron2} library to provide a pretrained detector.
We modified the detector to give raw softmax outputs for each object's class and raw sigmoid outputs for the binary masks.  
These became the functions $\hat{f}$ and $\hat{g}$ we described earlier.
We discretized the space, picking 
\begin{equation*}
    \Lambda = \underbrace{\{0.2,...,0.5\}}_{\text{\texttt{linspace(0.2,0.5,50)}}}\times\underbrace{\{0.3,...,0.7\}}_{\text{\texttt{linspace(0.3,0.7,5)}}}\times\underbrace{\{0.99,...,1\}}_{\text{\texttt{logspace(-0.00436,0,25)}}}.
\end{equation*}
Over 1000 random splits of the 5000 point validation set into a calibration set of 3000 points and a validation set of 2000 points, we produced violin plots of the three risks using the Bonferroni version of our method, with the target levels $\alpha = [0.5,0.5,0.75]$ and $\delta = 0.1$.
Although the levels $\alpha$ are nominally higher than they were in previous examples, they are quite stringent in this context.
For example, note that $R_3$ almost looks like coverage, but in practice is much more difficult to achieve because of the average over all predicted objects; if an object is predicted that does not match with a ground truth object, it counts as a miscoverage event.
So, the level $0.75$ is challenging.

Lastly, the reader may notice that the \texttt{recall} and \texttt{IOU} violin plots in Figure~\ref{fig:detection-histograms} are conservative.
This is due to the algorithm for choosing $\lhat$; we choose $\lambda_3$ first, $\lambda_1$ second, and $\lambda_2$ last.
By choosing $\lambda_3$ first, we told the algorithm to choose small prediction sets, then to optimize \texttt{recall} and \texttt{IOU} to be both controlled and as far from the nominal level as possible, provided the constraint on \texttt{coverage} holds.
This is why the \texttt{coverage} is controlled tightly, while the other risks are farther away from the constraint
(as shown in Figure~\ref{fig:detection-histograms}). 
Importantly, all risks are always controlled, so the calibrated detector can be used with confidence.

\section{Discussion}
\label{sec:discussion}

Our examples showcase useful statistical guarantees for complicated, modern machine learning systems without directly analyzing the fitting procedure.
Statistical procedures like LTT or conformal prediction, which wrap around machine learning deployments, are likely to increase in popularity---and even to become standard---since machine learning engineers will continue to build new models faster than case-by-case theoretical analyses can be completed.
Indeed, in the face of neural network engineering, statistics itself has become decoupled from how engineers train models; meanwhile, statistical ideas are all-the-more important for interpreting these non-parametric predictors.
We believe this tension will be a critical point of discussion for statisticians.

On the technical side, hypothesis testing combined with SGT allows for very tight control of risks while uniform concentration does not.
This is surprising; SGT was developed within the medical statistics community for very small hypothesis spaces, such as two major and two minor endpoints of a clinical trial.
Meanwhile, uniform concentration has a long history of use for the analysis of machine learning algorithms, and we derived sharp constants for the bound we tested using advanced uniform concentration inequalities~\citep[e.g.][]{anthony1993result}.
Despite this, in our work the testing framework is more useful, perhaps because it solves the easier discretized problem.

Turning to open questions, we first mention what has become a standard surgeon-general's warning in the field of distribution-free statistics: all the guarantees in this paper are marginal.
The reader should internalize what this means: we cannot guarantee that the errors are balanced over different strata of $X$- and $Y$-space, even meaningful ones such as object class, race, sex, illumination, et cetera.
All the errors may occur in one pathological bin---although it is possible to guard against such behaviors by designing a good score and evaluating the algorithms over relevant strata. Extending the proposed techniques to have errors exactly or approximately balanced across strata is an open direction of great importance.

The work presented herein opens many other questions for future work. 
We have not explored the question of optimal testing procedures---neither when they exist, nor what they would be, nor when they are SGTs.
Additionally, the SGTs used above are relatively simple and somewhat hand-designed heuristic procedures.
One could hope to do better by learning the graph from another data split.
We are not aware of work addressing this topic---such methods would not have been relevant in the medical statistics community that originated SGTs.
On a different note, the capability to tightly control non-monotonic risks should enable computer scientists and statisticians in many applied areas to more carefully audit and certify their algorithmic predictions.
We hope our examples demonstrate to this community that, even in complex setups with sophisticated algorithms, pragmatic and explicit statistical guarantees are possible.

\printbibliography

\clearpage
\appendix

\section{Experimental Details and Additional Results}\label{app:experiments}

\subsection{Numerical comparisons}
\label{app:theory-numerics}
We directly compare the numerical performance of the methods that we have presented in a synthetic first-order autoregressive (AR) process.
An AR process is a sequence of random variables where the elements in the sequence depend on the previous value plus a noise term.
Although the elements in the sequence are correlated, the mean of the process can be designed to have any shape.
Loosely, we model the loss of example $i \in \{1,...,n\}$ as an independently drawn AR process indexed by a one dimensional ordered set of $\Lambda$, with a ``V''-shaped mean.
In other words, the risk is ``V''-shaped, with the tip falling below $\alpha$.
We will make this description mathematically concrete in a moment.

In addition to the multiple testing strategies presented in the main text, we consider an alternative approach outlined in Appendix~\ref{app:theory-uniform}, using uniform concentration results to guarantee risk control. 
Unfortunately, this approach was seen to be more conservative than the multiple testing approach in the main-text examples.

Now, we turn to the concrete definition of our AR process.
We consider a case where $\Lambda = \{0.001, 0.002, \dots, 1\}$ and the true risk $R(\lambda)$ is ``V''-shaped, falling below the desired level $\alpha$ in the center of the interval; see Figure~\ref{fig:concentration_comparisons}. 
For each observation $i=1,\dots,n$, we
simulate losses for $\T_{\lambda_j}$ from the following first-order autoregressive model:
\begin{align}
    \label{eq:armodel}
    L_{i,j} &= \Phi(u_j + \mu_j) \\
    u_j &= \mathrm{corr} \cdot u_{j-1} + \sqrt{1-\mathrm{corr}^2} \cdot \mathcal{N}(0,1).
\end{align}
Here, the $\mu_j$ are chosen in such a way as to create the ``V''-shaped risk curve; note that there is no dependence on $i$ anywhere, because an independent random process gets generated for each sample $i \in \{1,...,n\}$.
From this simulated data, we form p-values for each hypothesis using the Hoeffding-Bentkus bound in~\eqref{eq:hb-p-value}. 
Then, we apply the Bonferroni test and the fixed sequence test to these p-values.

Before turning to the results, we pause to think about the meaning of our setting.
Here, we wish to find a value of $\lambda$ such that the risk is below $\alpha$, but do not know in advance where this may happen. 
Therefore, we use the data to test whether the risk is below $\alpha$ for each value of $\lambda$ using the calibration data. Note that because in practice all tests are based on the same calibration data, the results will be dependent, which is why our simulated setting includes the autoregressive correlation term.

The results are recorded in Figure~\ref{fig:concentration_comparisons}, which we parse next.
Firstly, the uniform concentration bound, perhaps predictably, performs badly in practice; see our discussion of this topic in Section~\ref{sec:discussion}. 
Secondly, the fixed sequence method is less conservative than the Bonferroni method, as expected; fixed sequence testing is designed to make the rejection set  as wide as possible on the right side.

\begin{figure}[H]
    \centering
    \begin{minipage}{0.45\textwidth}
        \includegraphics[width=0.7\linewidth]{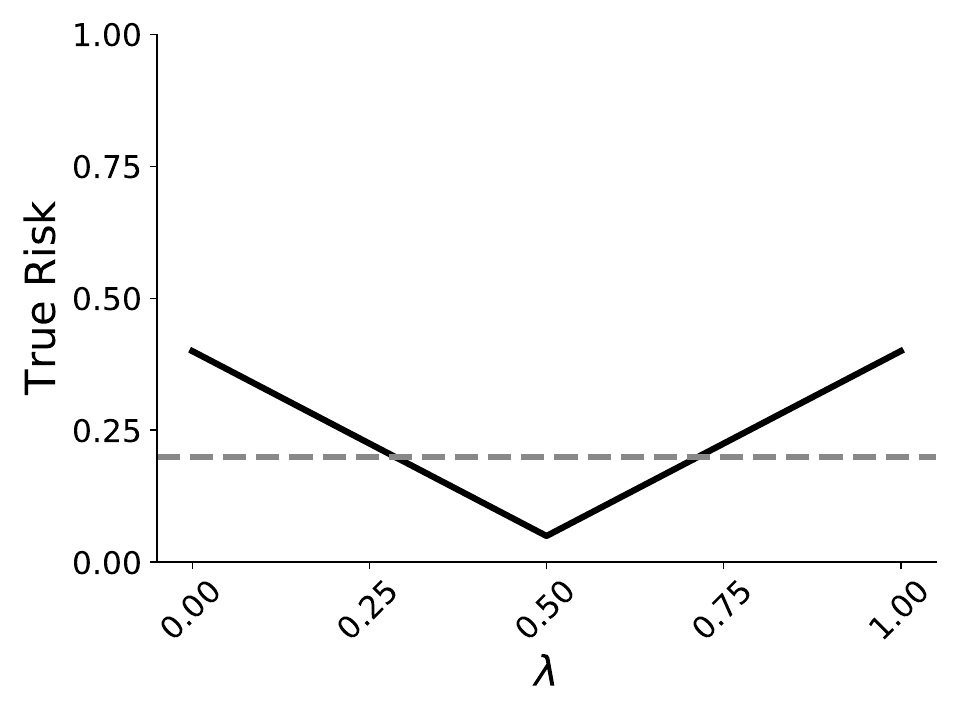}
    \end{minipage}
    \begin{minipage}{0.45\textwidth}
        \input{figures/tables/ar-process-table}
    \end{minipage}
    \caption{{\bf Numerical comparison of FWER-controlling algorithms on a synthetic AR process.} 
    On the left-hand side, we plot a representative of the ``V''-shaped risk used to generate the AR process, along with the desired risk level $\alpha=0.1$ as a gray dotted line.
    In the right table, we report the rightmost endpoint of the set $\Lhat$ for different FWER-controlling procedures (higher is less conservative).
    The subplots show different levels $\alpha$ and correlation parameters for the AR process, as in Section~\ref{eq:armodel}. An entry of $\emptyset$ means that the procedure failed to reject anywhere: $\Lhat = \emptyset$.
    We used $n=5000$ data points, $N=1000$ evenly spaced grid points of $\lambda$ between 0 and 1, $\mathrm{corr}=0.9$, and $\delta=0.1$.  The entry titled `Empirical risk $<\alpha$' in the table is a baseline procedure that selects the largest $\lambda$ such that the empirical risk is below $\alpha$. This baseline does not control the risk, and represents an upper bound on what we can hope to achieve with a risk-controlling procedure.}
    \label{fig:concentration_comparisons}
\end{figure}

\subsection{Additional experimental results for FDR control}
In Figure~\ref{fig:0_5_coco_histograms}, we report additional results from the experiment from Section~\ref{sec:multilabel}, this time with desired FDR level $\alpha = 0.5$.

\begin{figure}[H]
    \centering
    \begin{minipage}{0.4\textwidth}
        \centering
        \includegraphics[width=\linewidth]{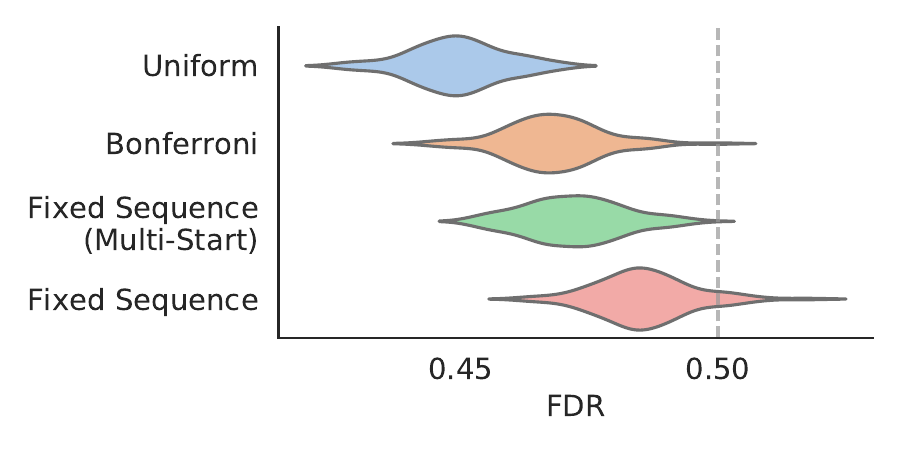}
    \end{minipage}
    \begin{minipage}{0.59\textwidth}
        \vspace{-0.5cm}
        \centering
        \input{figures/tables/0_5-coco-table}
    \end{minipage}
    \caption{{\bf Numerical results of our multi-label classification procedure on MS COCO.} This plot is the same figure as Figure~\ref{fig:0_2_coco_histograms} in the main text, but with $\alpha=0.5$.}
    \label{fig:0_5_coco_histograms}
\end{figure}

\subsection{Details about the ODIN score}
\label{app:odin-rescaling}
Next, we precisely define the ODIN function, which we use for as an OOD score.
Let $\hat{f}_y(x)$ be the output of the DenseNet's final, temperature-scaled~\cite{platt1999probabilistic} softmax layer, meant to estimate $\P\left(Y \mid X\right)$.
Additionally, let $\hat{f}_{(i)}(x)$ be the softmax output of the $i$-th most likely class.
The key step in ODIN is to add a small perturbation to the input image designed to help distinguish between in-distribution and out-of-distribution images.
In particular, the perturbation
\begin{equation}
    w(X) = X - \epsilon \,  \mathrm{sign}\left( -\nabla_X \log \hat{f}_{(1)}(X) \right)
\end{equation}
encourages the top softmax output to be more confident on the perturbed input $w(X)$.
The tuning parameter $\epsilon$ is the magnitude of the perturbation---we took it to be 0.0014, the default value from the ODIN GitHub repository.
Then, the OOD function is taken to be
\begin{equation}
    \mathrm{OOD}(X) = 1-\hat{f}(w(X))_{(1)}.
\end{equation}
The reader can view ODIN as an improvement upon the obvious OOD function, $1-\hat{f}(X)_{(1)}$; the improvement relies on the empirical observation that the top softmax output grows more for in-distribution images when perturbed than it does for out-of-distribution images.

We make one final note here about the scaling of the ODIN function.
With this particular model on the CIFAR-10 dataset, all values of $\mathrm{OOD}(X)$ were between the numbers $0.890$ and $0.899$.
In other words, this score is extremely poorly scaled; when choosing 1000 equally spaced points between 0 and 1000, only 9 of them will land in this interval.
We re-scaled ODIN manually to fix this problem, although it could be fixed using the training dataset.

\section{CLT p-value}
\label{app:clt-p-value}

Defining the empirical standard deviation as $\hat{\sigma}_j=((L_{1,j}-\Rhat_j)^2 + \ldots + (L_{n,j}-\Rhat_j)^2)^{1/2}/ (n-1)$, the p-value is
\begin{prop}[CLT p-values]
    \label{prop:clt-p-values}
    Suppose $L(T_{\lambda_j}(X_i),Y_i)$ has a finite mean and variance. Then
    \begin{equation}
        p_j^{\rm CLT} = 1-\Phi\left(\frac{\alpha - \Rhat_j}{\hat{\sigma}_j/\sqrt{n}}\right)
    \end{equation}
    is an asymptotically valid p-value: for all $u \in [0,1],$ $\limsup_{n\rightarrow \infty}\P(p_j^{\rm CLT} \le u) \le u$.
\end{prop}
The standard finite-sample Berry-Essen correction can be made, though we do not state it here.

\section{Proofs}
\label{app:proofs}
\begin{proof}[Proof of Theorem~\ref{thm:master-multiple-testing}]
    The result is immediate.
\end{proof}

\begin{proof}[Proof of Proposition~\ref{cor:hb-p-value}]
This is a combination of the results from~\cite{hoeffding1963} and~\cite{bentkus2004hoeffding}. See~\cite{bates2021distribution} for details on how to hybridize the two inequalities.
\end{proof}

\begin{proof}[Proof of Proposition~\ref{prop:clt-p-values}]
    This is a straightforward consequence of the central limit theorem.
\end{proof}

\begin{proof}[Proof of Proposition~\ref{prop:bonferroni}]
    This is a classical result in multiple testing. For example, see~\cite{holm1979simple} for discussion.
\end{proof}

\begin{proof}[Proof of Proposition~\ref{prop:bf_search}]
Fixed sequence testing has existed in the literature for some time; see, e.g.,~\cite{Sonnemann1986} and~\cite{bauer1991multiple}. We include a proof of the version we use for completeness.

Consider first the case where $|\mathcal{J}|$ = 1. Then there will be a first index at which you encounter a null. The probability of making a false discovery at that index is bounded by $\delta$. Thus, the probability of making any false discoveries is bounded by $\delta$. 

Turning to the $|\mathcal{J}| > 1$ case, the procedure is equivalent to running many instances of the $|\mathcal{J}| = 1$ procedure in parallel, at level $\delta / |\mathcal{J}|$. By the union bound, the probability of any false rejections is then bounded by $\delta$.
\end{proof}

\begin{proof}[Proof of Proposition~\ref{prop:fixed_sequence_J=1}]
  Since the procedure is sequential, it rejects a null iff it rejects $\mathcal{H}_{j^{*}}$, in which case $p_{j^{*}}\le \delta$.
\end{proof}

\begin{proof}[Proof of Proposition~\ref{prop:multirisk-pvalue}]
    The result is immediate.
\end{proof}

\begin{proof}[Proof of Proposition~\ref{thm:master-uniform}]
    By the definition of a uniform bound, $R^+$ fails with probability at most $\alpha$ at $\hat{\lambda}$.
\end{proof}

\section{Split Fixed Sequence Testing}
\label{app:split-fixed-sequence}
In Section~\ref{sec:sgt}, we showed how any pre-specified SGT can control the FWER.
The technique is flexible and powerful, especially for situations where it is straighforward to hand-design a graph, such as Figure~\ref{fig:ood-graph} in the OOD detection example.
However, for the large-scale machine learning systems we use, designing the graph is not always scaleable.
For example, consider the object detection example in Section~\ref{sec:detection}. 
Because all three coordinates of $\lambda$ affect all three risks, and we do not know \emph{a-priori} which values of $\lambda$ control the risks, there is no graph that is obviously best.
In such situations, we might desire an automated procedure that defines the graph by looking at promising hypotheses in our data; in other words, we seek to \emph{learn the graph}. 

The procedure we define here, which we call \emph{split fixed sequence testing}, is a very simple way of learning the graph from an extra data split.
The idea is inspired by the following fact: in many settings, we hope to identify points that violate several risk functions equally often.
Therefore, we pick a sequence of $\lambda$ that have nearly the same p-value using the first split of data.
Then, we simply use fixed sequence testing on this selected sequence using the fresh data.
This allows us to decide our path through $\Lambda$ using data while still providing rigorous FWER control.

Concretely, the algorithm begins by partitioning the calibration data points $\Ical_{\mathrm{cal}}$ into a graph selection set $\Ical_{\mathrm{graph}}$ and a multiple testing set $\Ical_{\mathrm{testing}}$.
We will learn our path through $\Lambda$ only by looking at $\Ical_{\mathrm{graph}}$; we will reserve the other split of data simply for running SGT.
Following the notation from Section~\ref{sec:multi-risk}, and calculating the p-values only on the graph selection set $\Ical_{\mathrm{graph}}$, we define the function
\begin{equation}
    \tilde{\lambda}(\beta) = \lambda_j\text{, where } j = \underset{j'}{\arg\min} \Big|\Big| [p_{j,1},...,p_{j,m}] - [\beta, ..., \beta]  \Big|\Big|_{\infty},
\end{equation}
where $\beta$ ranges from $0$ to $1$. The function $\tilde{\lambda}(\beta)$ picks a point where the p-values for all risks are nearly equal to $\beta$.
We parameterize our path by discretizing this function.
For some positive integer $D$, we set 
\begin{equation}
    \tilde{\lambda}_d = \tilde{\lambda}\left(\frac{d}{D}\right)\text{, for }d = 0,1,...,D.
\end{equation}
Note that in practice, this sequence can have repeated values, usually adjacent to one another; we remove them in practice, and ignore them for notational convenience.

Once we have the sequence $\{\tilde{\lambda}_d\}_{d=0}^D$, we do fixed sequence testing directly on the sequence as it is naturally ordered using the multiple testing set $\Ical_{\mathrm{testing}}$.
Figure~\ref{fig:detection-histograms} shows the result of this procedure in our object detection example.

As a remark on the theory of this example, once we select a set $\widehat{\Lambda}$ using split fixed sequence testing, we still have to select the specific element $\lhat \in \widehat{\Lambda}$ to report our results.
Because we use FWER control to pick $\widehat{\Lambda}$, we can optimize over the parameters however we wish.
Accordingly, we follow the following procedure (note that all coordinates of $\lambda$ depend on each other, and must be chosen jointly):
\begin{align}
    \lhat_3 & = \underset{\lambda \in \Lhat}{\min} \left\{\lambda_3 : \exists \lambda_1 \in \Lhat,  1-\lambda_1 < \lambda_3\right\}\\
    \lhat_1 & = \underset{\lambda \in \Lhat}{\max} \left\{\lambda_1 : \lambda_3 = \lhat_3\right\} \\
    \lhat_2 & = \max_{\lambda\in\Lhat} \left\{ \lambda_2 : \left[\lhat_1, \lambda_2, \lhat_3\right] \in \underset{\lambda \in \Lhat}{\arg \min} \; \hat{R}_2(\lambda) \right\} ,
\end{align}
which optimizes over $\Lhat$ to get meaningful prediction sets, many detections, and a high IOU.

\section{SGT methods used in Section~\ref{sec:ood}}
\label{app:hamming-sgt}
\begin{figure}[t]
    \centering
    \includegraphics[width=0.33\textwidth]{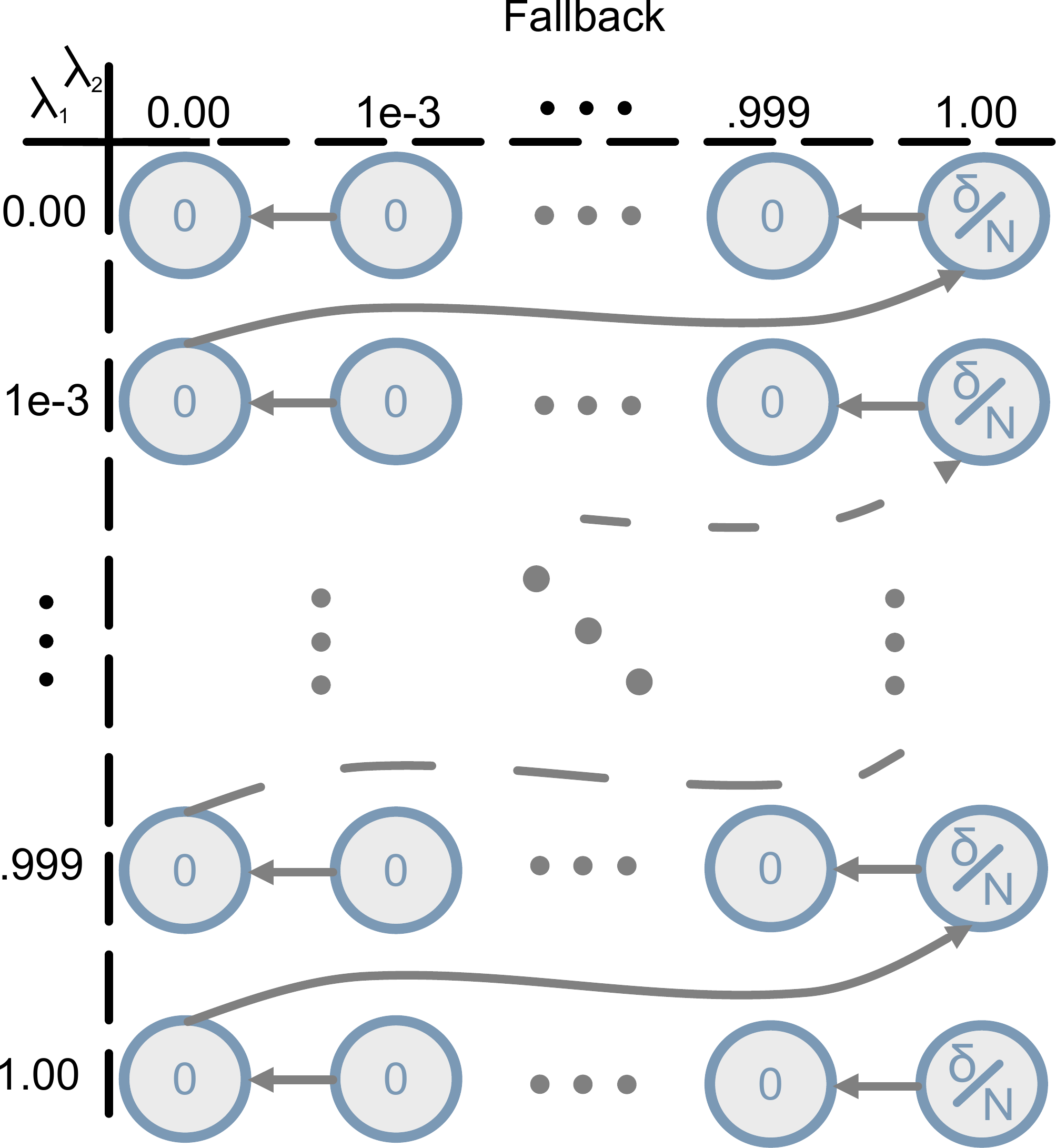}
    \hspace{1cm}
    \includegraphics[width=0.33\textwidth]{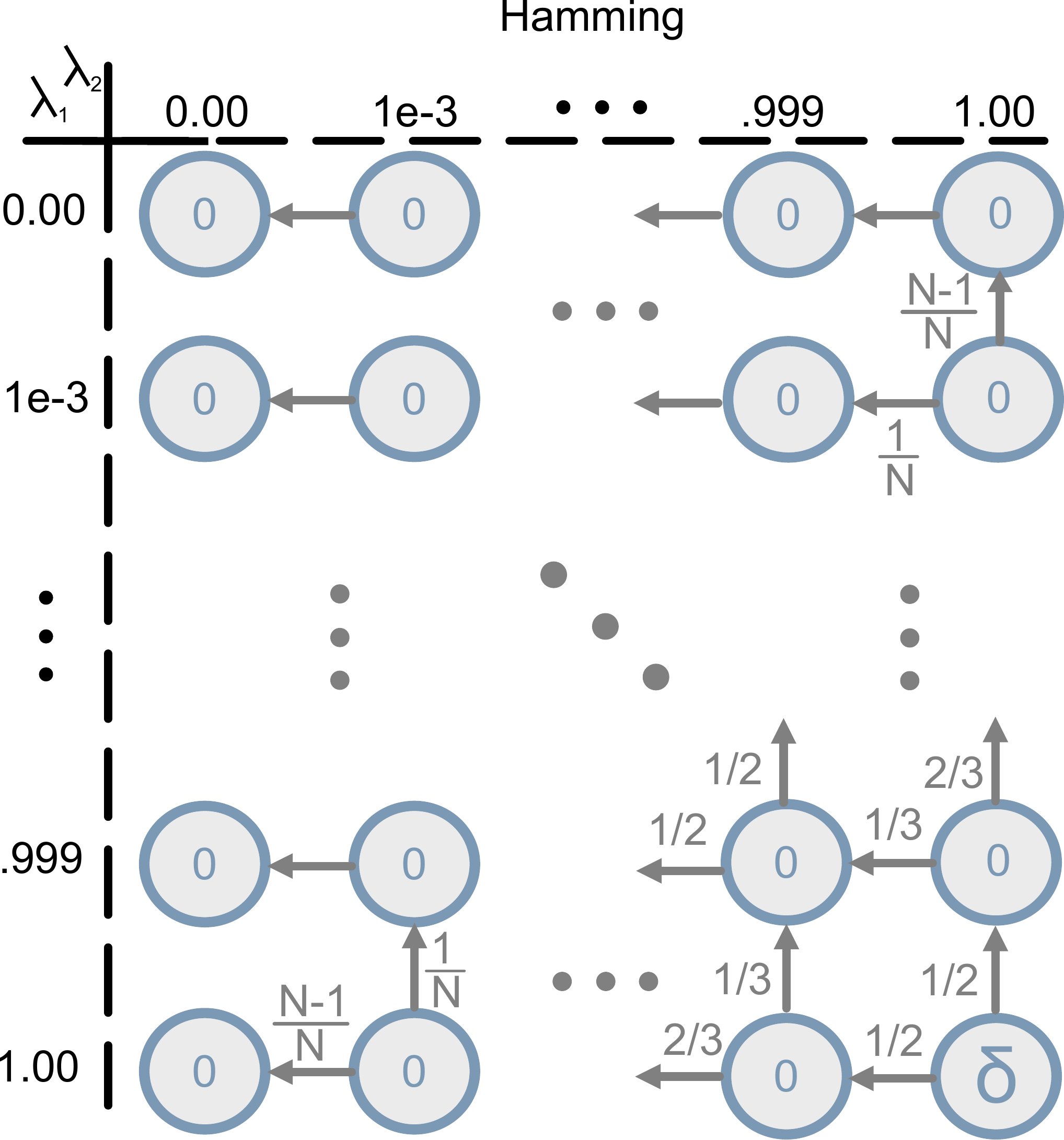}
    \caption{
    \textbf{The SGT graphs} for the two multiple testing procedures used in Section~\ref{sec:ood}.
    Each node corresponds to the null hypothesis $\mathcal{H}_{i,j} : R_1((\lambda_i,\lambda_j)) > \alpha_1 \text{ or } R_2((\lambda_i,\lambda_j)) > \alpha_2$.  
    The initial allocations of the error budget are shown within the nodes, and the edges define the transition matrix of the graph.
    Edges without a label have weight 1.
    The left graph describes a form of SGT called a \emph{fallback procedure}, wherein the hypotheses are chained together, with the initial error budget distributed throughout the chain.
    The right graph is the \emph{Hamming} graph, which allocates the entire error budget to the bottom right node and then propagates the error budget outwards; see Appendix~\ref{app:hamming-sgt}.
    }
    \label{fig:ood-graph}
\end{figure}

Recall that the SGT procedure from Section~\ref{sec:sgt} allows us to design any graph to propagate our error budget while still controlling type-I error.
We now turn to the details of the multiple testing methods used in Section~\ref{sec:ood}. First, we consider the SGT with graph shown in the left panel of Figure~\ref{fig:ood-graph}, which is known as the \emph{fallback procedure} \citep{wiens2003fixed, wiens2005fallback}.
Here, we start with the error budget spread equally across levels of $\lambda_1$, and then propagate the error budget towards the aforementioned Pareto frontier, starting from the safe points where $\lambda_2$ is large. 
This is nearly the same as the fixed sequence test that splits the error budget between each $\lambda_1$ and then running a fixed sequence test on each to find the smallest $\lambda_2$ for each $\lambda_1$, but it is more powerful due to the edges connecting adjacent rows.

Next, we consider SGT with the \emph{Hamming} graph in the right panel of Figure~\ref{fig:ood-graph}. 
Here, the initial error budget is all allocated to the bottom-right node, and then it cascades outward. 
The weights are chosen to yield a balanced procedure, in the sense that the maximum possible error budget a node can be allocated (which happens if there are rejections at all ancestor nodes) is $\delta$ divided by the Hamming distance to the bottom right.
Concretely, we associate the index $(i,j)$ to the node in the $i$th row from the bottom and the $j$th row from the right of the graph.
Then we set the weights of our graph to 
\begin{equation}
  W_{(i,j)\to(i,j+1)} =
    \begin{cases}
          \frac{j}{i+j} & i+j \leq N \\
          1 & i+j > N \text{ and } i = N \\
          0 & \text{ else,}
    \end{cases}
\end{equation}
and
\begin{equation}
  W_{(i,j)\to(i+1,j)} =
    \begin{cases}
          \frac{i}{i+j} & i+j \leq N \\
          1 & i+j > N \text{ and } i < N \\
          0 & \text{ else,}
    \end{cases}
\end{equation}
with all other edge weights equaling zero.
Setting these edge weights creates a graph where all nodes in the bottom-right triangle have \emph{balanced inflow}, i.e.,
\begin{definition}[Balanced Inflow]
\label{def:balanced_inflow}
  Let $i+j=L$ and consider the set of incoming paths to node $(i,j)$ as
\begin{equation}
  \mathcal{P}(i,j) = \Big\{ (i_1,j_1),(i_2,j_2),...,(i_{L-1},j_{L-1}),(i_L,j_L)\text{, where }(i_1,j_1)=(1,1)\text{ and }(i_L,j_L)=(i,j)\Big\}.
\end{equation}
  A graph has balanced inflow if 
  \begin{equation}
    \sum\limits_{\substack{\{(i_{i'},j_{i'})\}_{i'=1}^L\\ \in \mathcal{P}(i,j)}}\prod\limits_{l=1}^{L-1}W_{(i_l,j_l) \to (i_{l+1},j_{l+1})} = \frac{1}{L-1}. 
  \end{equation}
\end{definition}
In English, this property encodes the idea of balanced inflow we discussed earlier---nodes that are equally far away from the root can all potentially receive the same error budget (although in practice, rejections in the lower-right triangle may thwart this effort). Proving that the Hamming SGT has balanced inflow is simple by induction on $L$.
This graph structure plays well with coordinatewise monotonicity and other such settings where the goal is to push from one corner to another in the space of hypotheses. 

At a high level, the Hamming SGT improves on the fallback procedure when there are many rejections in the lower-right triangle.
Indeed, the Hamming SGT is much more powerful in this region, since the error budget is more concentrated.
In the best-case scenario, where all of the lower-right triangle is rejected, each node along the main diagonal gets a budget of $\frac{\delta}{N-1}$, and a fallback procedure is performed on the upper-left triangle.
However, the graph's structure can also hurt us; if a large fraction of nulls fail to be rejected in the lower-right triangle, the procedure could be less effective in rejecting an extreme setting of $\lambda$.
Thus, we suggest the use of the Hamming SGT if there is reason to believe that the lower-right triangle will contain many rejections.
We repeat that balanced inflow is not a necessary condition for a SGT graph to perform well, and in Section~\ref{sec:ood} we find that other approaches are more powerful.

As a final attempt on the multiple testing design, we designed a cascaded 2D fixed sequence test.
First, a budget of $\delta/2$ gets allocated to picking $\lambda_1$ via fixed sequence testing.
After finding the smallest $\lhat_1$ that controls $R_1$, we do a second fixed sequence test to find the smallest $\lambda_2$ such that the pair $(\lhat_1,\lambda_2)$ controls both risks.
We considered two hand-designed graphs in the main text: the fallback procedure and the so-called Hamming graph presented in Figure~\ref{fig:ood-graph}.
Now, we return to the Hamming graph, which improves upon the fallback procedure under certain conditions, and discuss its properties here.
In the Hamming graph, all of the error budget starts in the bottom-right corner, and propagates through two subparts of the graph.
In the lower-right triangle, the graph perfuses the error budget such that nodes with equal Hamming distance from the bottom-right node are designed to receive the same error budget if all previous nodes reject. We will call this property \emph{balanced inflow}; see Definition~\ref{def:balanced_inflow} below.
In the upper-left triangle, the graph propagates the error leftwards.

\section{Alternative Approach: Uniform Concentration}
\label{app:theory-uniform}
Rather than use multiple testing, one can consider creating risk-controlling predictions using uniform concentration results. We outline this strategy below. Unfortunately, this approach is typically much more conservative than the multiple testing approach, even when using state-of-the-art concentration results. This is evidenced in our many simulations. Nonetheless, since this is a seemingly natural alternative, we record this approach in detail here.

\subsection{Framework}
We stated how uniform bounds lead to risk control in Proposition~\ref{thm:master-uniform}.

The main challenge is to come up with a valid uniform confidence bound $R^+$. We have developed such bounds, stating the technical details in Appendix~\ref{app:uniform_bound_details}. 
In the experiments presented in the main text, we find that this is less powerful than the multiple-testing approach.

\subsection{A new class of concentration inequalities for self-normalized empirical processes}
\label{app:uniform_bound_details}
We next develop a state-of-the art set of uniform concentration results for our setting.
Recalling that 
\[\hat{R}(\T_{\lambda}) = \frac{1}{n}\sum_{i=1}^{n}L(\T_\lambda(X_i), Y_i),\]
we can view $\{\hat{R}(\T_\lambda): \lambda\in \Lambda\}$ as an empirical process indexed by $\lambda$. The simplest empirical process is the empirical CDF. The celebrated Dvoretzky-Kiefer-Wolfowitz-Massart (DKWM) inequality bounds the absolute difference $|\hat{F}_{n}(w) - F(w)|$. Although the constant is tight, the sup-difference metric ignores the non-constant variance of $\hat{F}_{n}(w)$. Specifically, $n\hat{F}_{n}(w)\sim \mathrm{Binom}(n, F(w))$ and thus $\Var[\hat{F}_{n}(w)] = F(w)(1 - F(w)) / n$. As a result, when $F(w)$ is close to $0$ or $1$, $\hat{F}_{n}(w)$ is more concentrated around $F(w)$. 

To improve the DKWM bounds for tail events, a line of work has established concentration inequalities for self-normalized empirical processes~\cite[e.g.][]{anthony1993result, hole1995some, bartlett1999inequality, bercu2002concentration, gine2003ratio}. In particular, they derived upper tail bounds for
 \[\sup_{w\in \R}\frac{\hat{F}_{n}(w) - F(w)}{\sqrt{\hat{F}_{n}(w)}} \text{ and }\sup_{w\in \R}\frac{F(w) - \hat{F}_{n}(w)}{\sqrt{F(w)}},\]
 where $\hat{F}_{n}(w)$ denotes the empirical CDF of $n$ i.i.d. samples $W_1, \ldots, W_n$ drawn from the distribution $F$. For instance,~\cite{anthony1993result} prove that, with probability $1 - \delta$,
 \[\sup_{w\in \R}\frac{F(w) - \hat{F}_{n}(w)}{\sqrt{F(w)}} \le \sqrt{\frac{4}{n}\log\lb\frac{8(n + 1)}{\alpha}\rb} = O\lb\sqrt{\frac{\log n}{n}}\rb.\]
 When $\hat{F}_{n}(w) = O(\log n / n)$, the Anthony--Shaw-Taylor inequality yields an upper confidence bound of $F(w)$ of order $O(\log n / n)$ as well, while the DKWM inequality implies an upper confidence bound of $F(w)$ of order $O(1 / \sqrt{n})$. 
 Thus, the former yields tighter results for rare events for large $n$.

In our context, $\hat{R}(\T_\lambda)$ is more complicated than the empirical CDF because the summands are no longer binary. We adapt the proofs of~\cite{anthony1993result},~\cite{parrondo1993vapnik}, and~\cite{vapnik1971uniform} to the general case and further improve their constants using the recently developed tools for sampling without replacement~\cite{bardenet2015concentration} and weighted sums of Rademacher variables~\cite{pinelis2012asymptotically, bentkus2015tight}. To state the general results, we consider a generic empirical process 
\begin{equation}\label{eq:generic_empirical_process}
\hat{s}_{n}(\lambda) = \frac{1}{n}\sum_{i=1}^{n}S(\lambda; Z_i), \quad s(\lambda) = \E[S(\lambda; Z_i)], \quad S(\lambda; Z_{i})\in [0, 1]\text{ almost surely for every }\lambda\in \Lambda.
\end{equation}
Note that $S(\lambda; Z_i) = L(T_\lambda(X_i), Y_i)$ in our context. Furthermore, we define $\Delta(n)$ as the 
\begin{equation}
  \label{eq:Deltan}
  \Delta(n) = \sup_{z_{1}, \ldots, z_{n}}\bigg|\big\{\{S(\lambda; z_1), \ldots, S(\lambda; z_n)\}: \lambda\in \Lambda\big\}\bigg|.
\end{equation}
In the literature, $\log \Delta(n)$ is often referred to as the growth function~\cite[][Section 2]{vapnik2013nature}.

 \begin{theorem}\label{thm:normalized_vapnik}
   Under the setting \eqref{eq:generic_empirical_process}, for any $\eta \ge 0$,
   \begin{equation}
     \label{eq:slam_lower}
     \P\lb \sup_{\lambda\in \Lambda}\frac{\hat{s}_{n}(\lambda) - s(\lambda)}{\sqrt{\hat{s}_n(\lambda) + \eta}}\ge t\rb\le  \inf_{\gamma\in (0,1), n'\in \Z^{+}}\frac{\Delta(n + n')\exp\{-g_2(t; n, n', \gamma, \eta)\}}{1 - \exp\{-g_1(t; n', \gamma, \kappa^{+})\}},
   \end{equation}
   and
   \begin{equation}
     \label{eq:slam_upper}
     \P\lb \sup_{\lambda\in \Lambda}\frac{s(\lambda) - \hat{s}_{n}(\lambda)}{\sqrt{s(\lambda) + \eta}}\ge t\rb\le \inf_{\gamma\in (0,1), n'\in \Z^{+}}\frac{\Delta(n + n')\exp\{-g_2(t; n, n', \gamma, \kappa^{-})\}}{1 - \exp\{-g_1(t; n', \gamma, \eta)\}},
   \end{equation}
   where 
   \begin{align*}
     g_1(t; n', \gamma, \kappa) &= \max\left\{\frac{n't^2}{2}\frac{\gamma^2}{1 + \gamma^2 t^2/36\kappa},\,\, \log\lb  \frac{n' t^2\gamma^2}{(\sqrt{1 + \kappa} - \sqrt{\kappa})^2}\rb\right\},\\
     g_2(t; n, n', \gamma, \kappa) &= \frac{nt^2}{2}\lb\frac{n'}{n + n'}\rb^2\frac{(1-\gamma)^2}{1 + (1-\gamma)^2 t^2/36\kappa},
   \end{align*}
   and
   \[\kappa^{+} = \eta + \frac{t^2}{2} + t\sqrt{\frac{t^2}{4} + \eta}, \quad \kappa^{-} = \eta + \frac{n + \gamma n'}{n + n'}\sqrt{\kappa^{+}}\]
\end{theorem}

\begin{theorem}\label{thm:normalized_vapnik_rademacher}
Under the setting \eqref{eq:generic_empirical_process}, for any $\eta \ge 0$,
  \begin{equation}
    \label{eq:slam_lower_rademacher}
     \P\lb \sup_{\lambda\in \Lambda}\frac{\hat{s}_{n}(\lambda) - s(\lambda)}{\sqrt{\hat{s}_n(\lambda) + \eta}}\ge t\rb\le \inf_{\gamma\in (0,1)}\frac{\Delta(2n)\td{g}\lb \sqrt{\frac{n(1 + \eta)}{2}}(1 - \gamma)t\rb}{1 - \exp\{-g_1(t; n, \gamma, \kappa^{+})\}},
   \end{equation}
   and
   \begin{equation}
     \label{eq:slam_upper_rademacher}
     \P\lb \sup_{\lambda\in \Lambda}\frac{s(\lambda) - \hat{s}_{n}(\lambda)}{\sqrt{s(\lambda) + \eta}}\ge t\rb\le \inf_{\gamma\in (0,1)}\frac{\Delta(2n)\td{g}\lb \sqrt{\frac{n(1 + \eta)}{2}}(1 - \gamma)t\rb}{1 - \exp\{-g_1(t; n, \gamma, \eta)\}},
   \end{equation}
   where $g_1$ is defined in Theorem \ref{thm:normalized_vapnik}, $\td{g}(x) = \min\{\td{g}_1(x), \td{g}_2(x), \td{g}_3(x)\}$, 
   \begin{align*}
     \td{g}_1(x) &= c_1(1 - \Phi(x)), \quad c_1 = 1 / 4(1 - \Phi(\sqrt{2}))\approx 3.178,\\
     \td{g}_2(x) &= 1 - \Phi(x) + \frac{c_2}{9 + x^2}\exp\left\{-\frac{x^2}{2}\right\}, \quad c_2 = 5\sqrt{e}(2\Phi(1) - 1)\approx 5.628,\\
     \td{g}_3(x) &= \exp\left\{-\frac{x^2}{2}\right\}.
   \end{align*}
 \end{theorem}

The proofs of both theorems are lengthy and relegated to Section \ref{subapp:proofs}.

\subsection{Uniform upper confidence bound for RCP}
From these general results, we obtain the following corollary, which will result in a practical algorithm momentarily.
\begin{cor}\label{cor:normalized_vapnik_UCB}
    Let 
    \begin{equation*}
        R^+(\T_\lambda) :=  \hat{R}(\T_\lambda) + t(\eta; \delta)\sqrt{\hat{R}(\T_\lambda) + \eta + \frac{t(\eta; \delta)^2}{4}} + \frac{t(\eta; \delta)^2}{2}.
    \end{equation*}
    where $t(\eta; \delta)$ is defined as the $t$ that solves
    \begin{equation*}
        \delta = \underset{\gamma \in (0,1), n' \in \mathbb{Z}^+}{\inf} \min\left\{\frac{ \Delta(n+n')\exp\big\{-g_2(t;n,n',\gamma,\kappa^-)\big\}}{1-\exp\{-g_1(t;n',\gamma,\eta)\}}, \frac{\Delta(2n)\td{g}\lb \sqrt{\frac{n(1 + \eta)}{2}}(1 - \gamma)t\rb}{1 - \exp\{-g_1(t; n, \gamma, \eta)\}}\right\},
    \end{equation*}
    where the functions and quantities are defined in Theorem \ref{thm:normalized_vapnik} and \ref{thm:normalized_vapnik_rademacher}. Then
    \[\P(R(\T_\lambda)\le R^{+}(\T_\lambda)\text{ for all }\lambda\in \Lambda)\ge 1 - \delta.\]
\end{cor}

\subsubsection*{An optimal choice of $\eta$}
The parameter $\eta$ needs to be chosen in advance. For a given level $(\alpha, \delta)$, however, there is an optimal choice. Let 
\begin{equation*}
    x(\eta; \delta) := \alpha - t(\eta; \delta)\sqrt{\alpha + \eta},
\end{equation*}
which is the largest solution $x$ of
\begin{equation*}
    \alpha = x + t(\eta; \delta)\sqrt{x + \eta + \frac{t(\eta; \delta)^2}{4}} + \frac{t(\eta; \delta)^2}{2}.
\end{equation*}
This is the largest value for which $\fdp(\T) \le x$ implies that $\fdr^+(\T) \le \alpha$. In other words, with $\eta$ fixed, our procedure will find the largest $\lambda$ such that $\fdp(\T_\lambda) \le x(\eta; \delta)$, and select this value as $\hat{\lambda}$. Thus, we should select the value of $\eta$ that makes $x(\eta; \delta)$ as small as possible;
\begin{equation*}
\eta^*(\delta) := \underset{\eta \ge 0}{\arg\min} \ x(\eta; \delta)
\end{equation*}
is the best value of $\eta$.

\subsection{Concrete instantiation for FDR control}

\begin{figure}[t]
    \centering
    \includegraphics[width=0.6\linewidth]{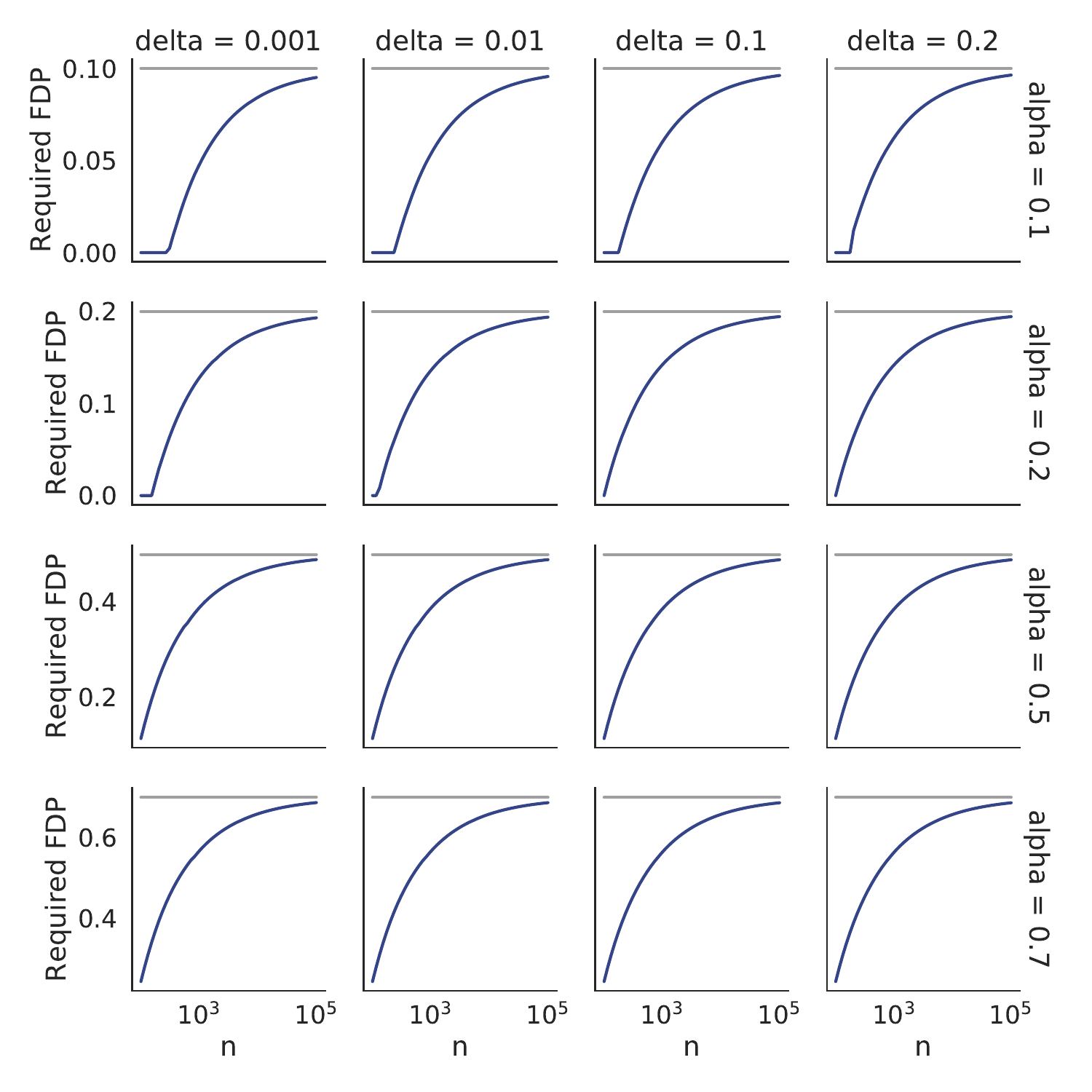}
    \vspace{-0.5cm}
    \caption{{\bf Empirical FDP needed to achieve FDR control.} We plot several desired $\fdr$ levels and high-probability choices $\delta$. The gray horizontal line indicates the target $\fdr$ level $\alpha$. }
    \label{fig:numerics}
\end{figure}

To apply Corollary~\ref{cor:normalized_vapnik_UCB} for FDR control, we need to derive the growth function. Let $Z_i = (X_i, Y_i)$, $\{1, \ldots, K\}$ be the set of labels, and 
\[\Lambda_{ij} := \inf\{\lambda \in \R: j \not\in \T_\lambda(X_i)\}\]
In our example, $\T_\lambda$ is decreasing in $\lambda$. Thus,
\[\fdr = \E[S(\lambda; Z_i)], \quad \text{where }S(\lambda; Z_i) = \frac{\sum_{j=1}^{m}I(\Lambda_{ij}\le \lambda, j\not \in Y_i)}{\sum_{j=1}^{m}I(\Lambda_{ij}\le \lambda)}.\]
Here, $0 / 0$ is defined as $0$. Note that $(S(\lambda; Z_1), \ldots, S(\lambda; Z_n))$ changes value only when $\lambda = \Lambda_{ij}$ for some $i$ and $j$, and thus
\begin{equation}
  \label{eq:Deltan_FDR}
  \Delta(n) \le nm + 1.
\end{equation}

In Algorithm~\ref{algo:bardenet} below, we present the full algorithm  corresponding to Corollary~\ref{thm:master-uniform}, specialized to the FDR control setting from Section~\ref{sec:multilabel} for concreteness.  In addition, in Figure~\ref{fig:numerics}, we report numerical information about the size of the bound. In particular, we show the empirical FDR that must be achieved in order to conclude that the true FDR is below $\alpha$ with probability at least $1-\delta$; that is, $x(\eta^*; \delta)$ in our notation above. When this curve is farther below the nominal rate $\alpha$, it means that the procedure is quite conservative. We find that the procedure is heavily conservative until $n$ is of order $10^5$. This provides further evidence that the uniform concentration approach is looser than the multiple testing approach, so it requires a much larger amount of calibration data.

\input{algorithms/bardenet}

\allowdisplaybreaks
\subsection{Technical proofs}\label{subapp:proofs}

\begin{proof}[Proof of Corollary~\ref{cor:normalized_vapnik_UCB}]
 Let $S(\lambda; Z_i) = L(T_\lambda(X_i), Y_i)$. By the second inequalities in Theorem \ref{thm:normalized_vapnik} and \ref{thm:normalized_vapnik_rademacher}, 
 \[\P\lb \sup_{\lambda\in \Lambda}\frac{R(\T_\lambda) - \hat{R}(\T_\lambda)}{\sqrt{R(\T_\lambda) + \eta}}\le t(\eta; \delta)\rb\ge 1 - \delta.\]
 On this event, for any $\lambda\in \Lambda$,
 \begin{align*}
     &R(\T_\lambda) - t(\eta; \delta)\sqrt{R(\T_\lambda) + \eta}\le \hat{R}(\T_\lambda)\\
     &\Longleftrightarrow \lb\sqrt{R(\T_\lambda) + \eta} - \frac{t(\eta; \delta)}{2}\rb^2 \le \hat{R}(\T_\lambda) + \eta + \frac{t(\eta; \delta)^2}{4}\\
     &\Longleftrightarrow \sqrt{R(\T_\lambda) + \eta}\le \sqrt{\hat{R}(\T_\lambda) + \eta + \frac{t(\eta; \delta)^2}{4}} + \frac{t(\eta; \delta)}{2}\\
     & \Longrightarrow R(\T_\lambda)\le \hat{R}(\T_\lambda) + t(\eta; \delta)\sqrt{\hat{R}(\T_\lambda) + \eta + \frac{t(\eta; \delta)^2}{4}} + \frac{t(\eta; \delta)^2}{2}.
 \end{align*}
 \end{proof}

 \begin{prop}\label{prop:bernstein}[Section 2.7 of~\cite{Boucheron2013}]
  Let $Z_1, \ldots, Z_n \in [0, 1]$ be i.i.d. random variables with $\E[Z_1] = \mu$ and $\Var[Z_1] = \sigma^2$. Further let $\hat{\mu} = (1/n)\sum_{i=1}^{n}Z_i$. Then for any $x > 0$,
   \[\P\lb \hat{\mu} - \mu \ge x\rb \le \exp\left\{-\frac{nx^2 / 2}{\sigma^2 + x / 3}\right\}.\]
 \end{prop}

  \begin{prop}\label{prop:bardenet_maillard}[~\cite{bardenet2015concentration}, Theorem 2.4]
    Let $x_1, \ldots, x_N$ be a fixed finite population of $N > 1$ real points with $x_i\in [0, 1]$ and $\bar{x} = (1 / N)\sum_{i=1}^{N}x_i$. Further let $\Pi$ be a random permutation of $\{1, \ldots, N\}$. Then for any $\eps > 0$,
    \[\P\lb \frac{1}{n}\sum_{k=1}^{n}x_{\pi(k)} - \bar{x}\ge \eps\rb\le \exp\left\{-\frac{2n\eps^2}{(1 - n/N)(1 + 1/n)}\right\}.\]
 \end{prop}

   \begin{prop}\label{prop:bardenet_maillard_bernstein}[~\cite{bardenet2015concentration}, Proposition 1.4]
     With the same setting as Proposition \ref{prop:bardenet_maillard}, for any $\eps > 0$,
     \[\P\lb \frac{1}{n}\sum_{k=1}^{n}x_{\pi(k)} - \bar{x}\ge \eps\rb\le \exp\left\{-\frac{n\eps^2 / 2}{\sigma^2 + \eps / 3}\right\},\]
     where
     \[\sigma^2 = \frac{1}{n}\sum_{k=1}^{n}(x_k - \bar{x})^2.\]
 \end{prop}

 \begin{prop}\label{prop:bentkus2015tight}[~\cite{bentkus2015tight}, Theorem 1.1]
   Let $\eps_{1}, \ldots, \eps_{n}$ be i.i.d. Rademacher random variables and $a = (a_1, \ldots, a_n)$ be a vector with $\|a\|_{2}\le 1$. Then
   \[\P\lb\sum_{i=1}^{n}a_{i}\eps_{i}\ge x\rb \le c_1(1 - \Phi(x))\]
   where $\Phi(x)$ is the CDF of the standard normal distribution and $c_1 = 1 / 4(1 - \Phi(\sqrt{2})) \approx 3.178$.
 \end{prop}

 \begin{prop}\label{prop:pinelis2012asymptotically}[~\cite{pinelis2012asymptotically}, Theorem 1.1]
   With the same assumptions and notation as in Proposition \ref{prop:bentkus2015tight},
      \[\P\lb\sum_{i=1}^{n}a_{i}\eps_{i}\ge x\rb \le 1 - \Phi(x) + \frac{c_2}{9 + x^2}\exp\left\{-\frac{x^2}{2}\right\},\]
      where $c_2 = 5\sqrt{e}(2\Phi(1) - 1)\approx 5.628$.
 \end{prop}

 \begin{proof}[\textbf{Proof of Theorem \ref{thm:normalized_vapnik}}]
Let $Z_{n+1}, \ldots, Z_{n+n'}$ be i.i.d. fresh samples drawn from the same distribution as $Z_1$ and 
\[\hat{s}_{n'}(\lambda) = \frac{1}{n'}\sum_{i=n+1}^{n+n'}S(\lambda; Z_{i}).\]
Since the proof is quite involved, we decompose it into four steps.

  ~\\
 \noindent  \textbf{Step 1 for \eqref{eq:slam_lower}:} we shall prove that
  \begin{equation}
    \label{eq:step1_lower}
    \P\lb \sup_{\lambda\in \R}\frac{\hat{s}_{n}(\lambda) - \hat{s}_{n'}(\lambda)}{\sqrt{\hat{s}_{n+n'}(\lambda) + \eta}}\ge (1 - \gamma)t\rb\ge \P\lb\sup_{\lambda\in \R}\frac{\hat{s}_{n}(\lambda) - s(\lambda)}{\sqrt{\hat{s}_n(\lambda) + \eta}}\ge t\rb\inf_{\lambda\in \R}\P\lb \frac{\hat{s}_{n'}(\lambda) - s(\lambda)}{\sqrt{s(\lambda) + \eta}}\le \gamma t\rb.
  \end{equation}
Consider the event that there exists $\lambda^{*}\in \R$ such that
  \begin{equation}
    \label{eq:step1_lower_event}
    \frac{\hat{s}_{n}(\lambda^{*}) - s(\lambda^{*})}{\sqrt{\hat{s}_{n}(\lambda^{*}) + \eta}}\ge t, \quad \frac{\hat{s}_{n'}(\lambda^{*}) - s(\lambda^{*})}{\sqrt{s(\lambda^{*}) + \kappa^{+}}}\le \gamma t.
  \end{equation}
  The first inequality of \eqref{eq:step1_lower_event} implies that
  \[\lb\hat{s}_{n}(\lambda^{*}) + \eta\rb - t\sqrt{\hat{s}_{n}(\lambda^{*}) + \eta}\ge \eta + s(\lambda^{*})\Longrightarrow \lb\sqrt{\hat{s}_{n}(\lambda) + \eta} - \frac{t}{2}\rb^2 \ge \frac{t^2}{4} + \eta + s(\lambda^{*}).\]
  Since $\hat{s}_{n}(\lambda) + \eta > 0$ and $\sqrt{t^2 / 4 + \eta} > t / 2$, we have
  \begin{align}
    \hat{s}_{n}(\lambda^{*}) + \eta&\ge \lb\frac{t}{2} + \sqrt{\frac{t^2}{4} + \eta + s(\lambda^{*})}\rb^2 =  \frac{t^2}{2} + \eta + s(\lambda^{*}) + t\sqrt{\frac{t^2}{4} + \eta + s(\lambda^{*})}\nonumber\\
    \ge & s(\lambda^{*}) + \eta + \frac{t^2}{2} + t\sqrt{\frac{t^2}{4} + \eta} = s(\lambda^{*}) + \kappa^{+}.\label{eq:step1_lower_sn_s}
  \end{align}
  The second inequality of \eqref{eq:step1_lower_event} and \eqref{eq:step1_lower_sn_s} imply that
  \begin{equation}\label{eq:step1_lower_sn_snp}
    \frac{\hat{s}_{n'}(\lambda^{*}) - s(\lambda^{*})}{\sqrt{s(\lambda^{*}) + \kappa^{+}}}\le \gamma t \le t \le \frac{\hat{s}_{n}(\lambda^{*}) - s(\lambda^{*})}{\sqrt{\hat{s}_{n}(\lambda^{*}) + \eta}} \le \frac{\hat{s}_{n}(\lambda^{*}) - s(\lambda^{*})}{\sqrt{s(\lambda^{*}) + \kappa^{+}}}\Longrightarrow \hat{s}_{n}(\lambda^{*})\ge \hat{s}_{n'}(\lambda^{*}).
  \end{equation}
  As a result,
  \begin{equation}\label{eq:step1_lower_snnp_sn}
    \hat{s}_{n+n'}(\lambda) = \frac{n'}{n + n'}\hat{s}_{n}(\lambda) + \frac{n}{n + n'}\hat{s}_{n'}(\lambda)\le \hat{s}_{n}(\lambda)
  \end{equation}
  By \eqref{eq:step1_lower_event} - 
  \eqref{eq:step1_lower_snnp_sn},
  \begin{align*}
    \frac{\hat{s}_{n}(\lambda^{*}) - \hat{s}_{n'}(\lambda^{*})}{\sqrt{\hat{s}_{n+n'}(\lambda^{*}) + \eta}}
    & \stackrel{\eqref{eq:step1_lower_sn_snp} \text{ and }\eqref{eq:step1_lower_snnp_sn}}{\ge} \frac{\hat{s}_{n}(\lambda^{*}) - \hat{s}_{n'}(\lambda^{*})}{\sqrt{\hat{s}_{n}(\lambda^{*}) + \eta}}\\
    & \stackrel{\eqref{eq:step1_lower_event}}{\ge} \frac{t\sqrt{\hat{s}_{n}(\lambda^{*}) + \eta} - (\hat{s}_{n'}(\lambda^{*}) - s(\lambda^{*}))}{\sqrt{\hat{s}_{n}(\lambda^{*}) + \eta}}\\
    & \stackrel{\eqref{eq:step1_lower_event}}{\ge} \frac{t\sqrt{\hat{s}_{n}(\lambda^{*}) + \eta} - \gamma t\sqrt{s(\lambda^{*}) + \kappa^{+}}}{\sqrt{\hat{s}_{n}(\lambda^{*}) + \eta}}\\
    & \stackrel{\eqref{eq:step1_lower_sn_s}}{\ge} \frac{t\sqrt{\hat{s}_{n}(\lambda^{*}) + \eta} - \gamma t\sqrt{\hat{s}_{n}(\lambda^{*}) + \eta}}{\sqrt{\hat{s}_{n}(\lambda^{*}) + \eta}}\\
    & = (1 - \gamma)t.
  \end{align*}
  Given $(Z_1, \ldots, Z_n)$, if $\lambda\mapsto \frac{\hat{s}_{n}(\lambda) - s(\lambda)}{\sqrt{\hat{s}_{n}(\lambda) + \eta}}$ achieves the supremum, we take $\lambda^{*}$ as the maximizer. Then $\lambda^{*}$ is measurable with respect to $\{Z_1, \ldots, Z_n\}$ and independent of $\hat{s}_{n'}$. As a result,
  \begin{align*}
    \lefteqn{\P\lb \sup_{\lambda\in \R}\frac{\hat{s}_{n}(\lambda) - \hat{s}_{n'}(\lambda)}{\sqrt{\hat{s}_{n+n'}(\lambda) + \eta}}\ge (1 - \gamma)t\rb}\\ & \ge \P\lb \frac{\hat{s}_{n}(\lambda^{*}) - \hat{s}_{n'}(\lambda^{*})}{\sqrt{\hat{s}_{n+n'}(\lambda^{*}) + \eta}}\ge (1 - \gamma)t\rb\\
    &\ge \P\lb\frac{\hat{s}_{n}(\lambda^{*}) - s(\lambda^{*})}{\sqrt{\hat{s}_{n}(\lambda^{*}) + \eta}}\ge t, \,\, \frac{\hat{s}_{n'}(\lambda^{*}) - s(\lambda^{*})}{\sqrt{s(\lambda^{*}) + \kappa^{+}}}\le \gamma t\rb\\
    & = \E\left[\P\lb\frac{\hat{s}_{n}(\lambda^{*}) - s(\lambda^{*})}{\sqrt{\hat{s}_{n}(\lambda^{*}) + \eta}}\ge t, \,\, \frac{\hat{s}_{n'}(\lambda^{*}) - s(\lambda^{*})}{\sqrt{s(\lambda^{*}) + \kappa^{+}}}\le \gamma t\rb\mid Z_1, \ldots, Z_n\right]\\    
    & = \E\left[\ind{\frac{\hat{s}_{n}(\lambda^{*}) - s(\lambda^{*})}{\sqrt{\hat{s}_{n}(\lambda^{*}) + \eta}}\ge t} \P\lb\frac{\hat{s}_{n'}(\lambda^{*}) - s(\lambda^{*})}{\sqrt{s(\lambda^{*}) + \kappa^{+}}}\le \gamma t\mid Z_1, \ldots, Z_n\rb\right]\\
    & \ge  \P\lb\frac{\hat{s}_{n}(\lambda^{*}) - s(\lambda^{*})}{\sqrt{\hat{s}_{n}(\lambda^{*}) + \eta}}\ge t\rb \inf_{\lambda\in \R}\P\lb\frac{\hat{s}_{n'}(\lambda) - s(\lambda)}{\sqrt{s(\lambda) + \kappa^{+}}}\le \gamma t\rb\\
    & = \P\lb\sup_{\lambda\in \R}\frac{\hat{s}_{n}(\lambda) - s(\lambda)}{\sqrt{\hat{s}_{n}(\lambda) + \eta}}\ge t\rb \inf_{\lambda\in \R}\P\lb\frac{\hat{s}_{n'}(\lambda) - s(\lambda)}{\sqrt{s(\lambda) + \kappa^{+}}}\le \gamma t\rb.
  \end{align*}
  This concludes \eqref{eq:step1_lower}. If the supremum of $\lambda\mapsto \frac{\hat{s}_{n}(\lambda) - s(\lambda)}{\sqrt{\hat{s}_{n}(\lambda) + \eta}}$ cannot be achieved, we can find a sequence of events $\{\lambda_{\ell}: \ell = 1, 2, \ldots\}$ at which the values converge to the supremum. For each $\lambda_{\ell}$, we can prove the above inequality and \eqref{eq:step1_lower} can be proved by taking $\ell\rightarrow \infty$. 

  ~\\
  \noindent \textbf{Step 2 for \eqref{eq:slam_lower}:} we shall prove that
  \begin{equation}
    \label{eq:step2_lower}
    \inf_{\lambda\in \R}\P\lb\frac{\hat{s}_{n'}(\lambda) - s(\lambda)}{\sqrt{s(\lambda) + \kappa^{+}}}\le \gamma t\rb \ge 1 - \exp\{-g_1(t; n', \gamma, \kappa^{+})\}.
  \end{equation}
  Given any subset $\lambda\in \R$,
  \[\hat{s}_{n'}(\lambda) - s(\lambda) = \frac{1}{n'}\sum_{i=n+1}^{n+n'}\lb S(\lambda; Z_i) - s(\lambda)\rb,\]
  and
  \[\E[S(\lambda; Z_i) - s(\lambda)] = 0, \quad \E[(S(\lambda; Z_i) - s(\lambda))^2] \le \E[S(\lambda; Z_i)^2] \le \E[S(\lambda; Z_i)] = s(\lambda).\]
  By the Bernstein inequality (Proposition \ref{prop:bernstein}),
  \[\P\lb \hat{s}_{n'}(\lambda) - s(\lambda)\ge \gamma t \sqrt{s(\lambda) + \kappa^{+}}\rb\le \exp\left\{-\frac{n't^2}{2}\gamma^2\frac{s(\lambda) + \kappa^{+}}{s(\lambda) + \gamma t \sqrt{s(\lambda) + \kappa^{+}} / 3}\right\}.\]
  It remains to prove that
  \begin{equation}
    \label{eq:step2_lower_extra_term}
    \frac{s(\lambda) + \kappa^{+}}{s(\lambda) + \gamma t \sqrt{s(\lambda) + \kappa^{+}} / 3}\ge \lb 1 + \frac{\gamma^2 t^2}{36\kappa^{+}}\rb^{-1}.
  \end{equation}
  Using the fact that $\sqrt{a} \le (ba + 1 / b) / 2$, we have
  \[\sqrt{s(\lambda) + \kappa^{+}}\le \frac{\gamma t}{12\kappa^{+}}(s(\lambda) + \kappa^{+}) + \frac{3\kappa^{+}}{\gamma t}.\]
  This entails that
  \[\inf_{\lambda\in \R}\P\lb\frac{\hat{s}_{n'}(\lambda) - s(\lambda)}{\sqrt{s(\lambda) + \kappa^{+}}}\le \gamma t\rb \ge 1 - \exp\{-g_{11}(t; n', \gamma, \kappa^{+})\},\]
  where
  \[g_{11}(t; n', \gamma, \kappa^{+}) = \frac{n't^2}{2}\frac{\gamma^2}{1 + \gamma^2 t^2/36\kappa}.\]
  On the other hand, since $\E[\hat{s}_{n'}(\lambda)] = s(\lambda)$, by Chebyshev's inequality,
  \[\P\lb \frac{\hat{s}_{n'}(\lambda) - s(\lambda)}{\sqrt{s(\lambda) + \kappa^{+}}}\ge \gamma t\rb\le \frac{1}{\gamma^2 t^2}\Var\lb \frac{\hat{s}_{n'}(\lambda) - s(\lambda)}{\sqrt{s(\lambda) + \kappa^{+}}}\rb = \frac{1}{n'\gamma^2 t^2}\frac{s(\lambda)(1 - s(\lambda))}{s(\lambda) + \kappa^{+}}.\]
  Let $m(x) = x(1 - x) / (x + \kappa^{+})$. Then
  \[\frac{d}{dx}\log[m(x)] = \frac{1}{x} - \frac{1}{1 - x} - \frac{1}{x + \kappa^{+}} = \frac{\kappa^{+} - 2\kappa^{+} x - x^2}{x(x + \kappa^{+})(1 - x)}.\]
  Since $\eta \ge 0$, $\kappa^{+}\ge 0$. Via some tedious algebra, we can show that $m(x)$ achieves its maximum at $x^{*} = \sqrt{\kappa^{+} + \kappa^{+2}} - \kappa^{+}$ at which 
  \[m(x^{*}) = (\sqrt{1 + \kappa^{+}} - \sqrt{\kappa^{+}})^2.\]
  Therefore,
  \[\P\lb \frac{\hat{s}_{n'}(\lambda) - s(\lambda)}{\sqrt{s(\lambda) + \kappa^{+}}}\le \gamma t\rb\ge 1 - \exp\{-g_{12}(t; n', \gamma, \kappa^{+})\},\]
  where
  \[g_{12}(t; n', \gamma, \kappa^{+}) = \log\lb n'\gamma^2 t^2 / (\sqrt{1 + \kappa^{+}} - \sqrt{\kappa^{+}})^2\rb.\]
  Putting two pieces together, \eqref{eq:step2_lower} is proved by noting that $g_{1} = g_{11}\wedge g_{12}$.

  ~\\
  \noindent \textbf{Step 3 for \eqref{eq:slam_lower}:} we shall prove that
  \begin{equation}
    \label{eq:step3_lower}
    \P\lb \sup_{\lambda\in \R}\frac{\hat{s}_{n}(\lambda) - \hat{s}_{n'}(\lambda)}{\sqrt{\hat{s}_{n+n'}(\lambda) + \eta}}\ge (1 - \gamma)t\rb\le \Delta(n + n')\exp\{-g_2(t; n, n', \gamma, \eta)\}.
  \end{equation}
  Let $\Pi$ be any given permutation over $\{1, \ldots, n + n'\}$. Since $Z_{1}, \ldots, Z_{n+n'}$ are i.i.d.,
  \[(Z_{1}, \ldots, Z_{n+n'})\stackrel{d}{=}(Z_{\Pi(1)}, \ldots, Z_{\Pi(n+n')}).\]
  As a result, for any $\Pi$,
  \[\P\lb \sup_{\lambda\in \R}\frac{\hat{s}_{n}(\lambda) - \hat{s}_{n'}(\lambda)}{\sqrt{\hat{s}_{n+n'}(\lambda) + \eta}}\ge (1 - \gamma)t\rb = \P\lb \sup_{\lambda\in \R}\frac{\hat{s}_{n, \Pi}(\lambda) - \hat{s}_{n', \Pi}(\lambda)}{\sqrt{\hat{s}_{n+n'}(\lambda) + \eta}}\ge (1 - \gamma)t\rb\]
  where $\hat{s}_{n, \Pi}(\lambda)$ is the empirical CDF of $Z_{\Pi(1)}, \ldots, Z_{\Pi(n)}$ and $\hat{s}_{n', \Pi}(\lambda)$ is the empirical CDF of $Z_{\Pi(n+1)}, \ldots,$ $Z_{\Pi(n + n')}$. Note that $\hat{s}_{n+n'}(\lambda)$ is invariant with respect to $\Pi$. With a slight abuse of notation, we take $\Pi$ as a uniform permutation over $\{1, \ldots, n + n'\}$. Then
  \[\P\lb \sup_{\lambda\in \R}\frac{\hat{s}_{n}(\lambda) - \hat{s}_{n'}(\lambda)}{\sqrt{\hat{s}_{n+n'}(\lambda) + \eta}}\ge (1 - \gamma)t\rb = \E_{\hat{s}_{n+n'}}\left[\P_{\Pi}\lb \sup_{\lambda\in \R}\frac{\hat{s}_{n, \Pi}(\lambda) - \hat{s}_{n', \Pi}(\lambda)}{\sqrt{\hat{s}_{n+n'}(\lambda) + \eta}}\ge (1 - \gamma)t\mid \hat{s}_{n+n'}\rb\right].\]
  Let $\Gamma(n+n')$ be the collection of distinct sets in the form of
  \[\{S(\lambda; Z_1), \ldots, S(\lambda; Z_{n+n'}): \lambda\in \R\}\]
  It is easy to see that $|\Gamma(n+n')|\le \Delta(n+n')$. Then
  \begin{align*}
    \lefteqn{\P_{\Pi}\lb \sup_{\lambda\in \R}\frac{\hat{s}_{n, \Pi}(\lambda) - \hat{s}_{n', \Pi}(\lambda)}{\sqrt{\hat{s}_{n+n'}(\lambda) + \eta}}\ge (1 - \gamma)t\mid \hat{s}_{n+n'}\rb}\\
    & \le \sum_{\lambda\in \Gamma(n + n')}\P_{\Pi}\lb \frac{\hat{s}_{n, \Pi}(\lambda) - \hat{s}_{n', \Pi}(\lambda)}{\sqrt{\hat{s}_{n+n'}(\lambda) + \eta}}\ge (1 - \gamma)t\mid \hat{s}_{n+n'}\rb.
  \end{align*}
  It remains to prove that
  \begin{equation}
    \label{eq:step3_lower_goal}
    \P_{\Pi}\lb \frac{\hat{s}_{n, \Pi}(\lambda) - \hat{s}_{n', \Pi}(\lambda)}{\sqrt{\hat{s}_{n+n'}(\lambda) + \eta}}\ge (1 - \gamma)t\mid \hat{s}_{n+n'}\rb\le \exp\{-g_2(t; n, n', \gamma, \eta)\} \quad \mbox{a.s.}.
  \end{equation}
  By definition,
  \[\hat{s}_{n, \Pi}(\lambda) - \hat{s}_{n', \Pi}(\lambda) = \hat{s}_{n, \Pi}(\lambda) - \frac{1}{n'}\lb (n + n')\hat{s}_{n+n'}(\lambda) - n\hat{s}_{n, \Pi}(\lambda)\rb = \frac{n + n'}{n'}(\hat{s}_{n, \Pi}(\lambda) - \hat{s}_{n + n'}(\lambda)).\]
  Note that $n\hat{s}_{n}(\lambda)$ is the sum of $n$ elements from the unordered set $\{S(\lambda; Z_1), \ldots, S(\lambda; Z_{n+n'})\}$ sampled without replacement. By Proposition \ref{prop:bardenet_maillard_bernstein},
  \begin{align*}
    \lefteqn{\P\lb \hat{s}_{n}(\lambda) - \hat{s}_{n + n'}(\lambda)\ge \frac{n'}{n + n'}(1 - \gamma)t\sqrt{\hat{s}_{n+n'}(\lambda) + \eta}\mid \hat{s}_{n + n'}\rb}\\
    & \le \exp\left\{-\frac{nt^2}{2}\lb\frac{n'}{n + n'}\rb^2(1 - \gamma)^2\frac{\hat{s}_{n + n'}(\lambda) + \eta}{\hat{s}_{n + n'}(\lambda) + (1 - \gamma) t \frac{n'}{n + n'}\sqrt{\hat{s}_{n + n'}(\lambda) + \eta} / 3}\right\}\\
    & \le \exp\left\{-\frac{nt^2}{2}\lb\frac{n'}{n + n'}\rb^2(1 - \gamma)^2\frac{\hat{s}_{n + n'}(\lambda) + \eta}{\hat{s}_{n + n'}(\lambda) + (1 - \gamma) t\sqrt{\hat{s}_{n + n'}(\lambda) + \eta} / 3}\right\}\\
  \end{align*}
  Using the same argument as \eqref{eq:step2_lower_extra_term}, we can prove that
  \[\frac{\hat{s}_{n + n'}(\lambda) + \eta}{\hat{s}_{n + n'}(\lambda) + (1 - \gamma) t \sqrt{\hat{s}_{n + n'}(\lambda) + \eta} / 3}\ge \lb 1 + \frac{(1 - \gamma)^2 t^2}{36\eta}\rb^{-1}.\]
  Therefore,
  \[\P\lb \hat{s}_{n}(\lambda) - \hat{s}_{n + n'}(\lambda)\ge \frac{n'}{n + n'}t\sqrt{\hat{s}_{n+n'}(\lambda) + \eta}\mid \hat{s}_{n + n'}\rb \le \exp\{-g_2(t; n, n', \gamma, \eta)\}.\]
  Since the bound is independent of $\hat{s}_{n + n'}$, \eqref{eq:step3_lower_goal} is proved and thus step 3.

  ~\\
  \noindent \textbf{Step 4 for \eqref{eq:slam_lower}:} putting \eqref{eq:step1_lower}, \eqref{eq:step2_lower} and \eqref{eq:step3_lower} together, we prove that for any $\gamma \in (0, 1)$ and $n' \in \Z^{+}$,
  \[\P\lb \sup_{\lambda\in \R}\frac{\hat{s}_{n}(\lambda) - s(\lambda)}{\sqrt{\hat{s}_{n}(\lambda) + \eta}}\ge (1 - \gamma)t\rb\le \frac{\Delta(n + n')\exp\{-g_2(t; n, n', \gamma, \eta)\}}{1 - \exp\{-g_1(t; n', \gamma, \kappa^{+})\}}.\]
  Since the right-handed side is deterministic, we can take infimum over $\gamma$ and $n'$, which yields \eqref{eq:slam_lower}.

  ~\\
  \noindent To prove \eqref{eq:slam_upper}, we follow the same steps as above.

  ~\\
  \noindent \textbf{Step 1 for \eqref{eq:slam_upper}:} we shall prove that
  \begin{equation}
    \label{eq:step1_upper}
    \P\lb \sup_{\lambda\in \R}\frac{\hat{s}_{n'}(\lambda) - \hat{s}_{n}(\lambda)}{\sqrt{\hat{s}_{n + n'}(\lambda) + \kappa^{-}}}\ge (1 - \gamma)t\rb\ge \P\lb\sup_{\lambda\in \R}\frac{s(\lambda) - \hat{s}_{n}(\lambda)}{\sqrt{s(\lambda) + \eta}}\ge t\rb\inf_{\lambda\in \R}\P\lb \frac{s(\lambda) - \hat{s}_{n'}(\lambda)}{\sqrt{s(\lambda) + \eta}}\le \gamma t\rb.
  \end{equation}
  Consider the event that there exists $\lambda^{*}\in \R$ such that
  \begin{equation}
    \label{eq:step1_upper_event}
    \frac{s(\lambda^{*}) - \hat{s}_{n}(\lambda^{*})}{\sqrt{s(\lambda^{*}) + \eta}}\ge t, \quad \frac{s(\lambda^{*}) - \hat{s}_{n'}(\lambda^{*})}{\sqrt{s(\lambda^{*}) + \eta}}\le \gamma t
  \end{equation}
  As with \eqref{eq:step1_upper_sn_s}, we can show that
  \begin{equation}
    \label{eq:step1_upper_sn_s}
    s(\lambda^{*}) + \eta \ge \hat{s}_{n}(\lambda^{*}) + \kappa^{+}.
  \end{equation}
  On the other hand, by \eqref{eq:step1_upper},
  \begin{equation}
    \label{eq:step1_upper_sn_snp}
    \hat{s}_{n'}(\lambda^{*})\ge s(\lambda^{*}) - \gamma t\sqrt{s(\lambda^{*}) + \eta} \ge s(\lambda^{*}) - t\sqrt{s(\lambda^{*}) + \eta} \ge \hat{s}_{n}(\lambda^{*}).
  \end{equation}
  Let $a, b, c$ be arbitrary positive numbers and $d(x) = (x - a) / \sqrt{bx + c}$. Then
  \[\frac{d}{dx}\log d(x) = \frac{1}{x - a} - \frac{b}{2(bx + c)} = \frac{bx + 2c + ab}{2(x-a)(bx + c)}.\]
  Thus, $d(x)$ is increasing on $[a, \infty)$. Take $a = \hat{s}_{n}(\lambda)$, $b = n' / (n + n')$, and $c = n a / (n + n') + \kappa^{-}$. By \eqref{eq:step1_upper_sn_snp},
  \begin{align*}
    \frac{\hat{s}_{n'}(\lambda^{*}) - \hat{s}_{n}(\lambda^{*})}{\sqrt{\hat{s}_{n + n'}(\lambda^{*}) + \kappa^{-}}}
    & = d\lb\hat{s}_{n'}(\lambda^{*})\rb \ge d\lb s(\lambda^{*}) - \gamma t\sqrt{s(\lambda^{*}) + \eta}\rb\\
    & = \frac{s(\lambda^{*}) - \gamma t\sqrt{s(\lambda^{*}) + \eta} - \hat{s}_{n}(\lambda^{*})}{\sqrt{\frac{n'}{n + n'}(s(\lambda^{*}) - \gamma t\sqrt{s(\lambda^{*}) + \eta}) + \frac{n}{n + n'}\hat{s}_{n}(\lambda^{*}) + \kappa^{-}}}\\
    & \stackrel{\eqref{eq:step1_upper_sn_s}}{\ge}\frac{s(\lambda^{*}) - \gamma t\sqrt{s(\lambda^{*}) + \eta} - \hat{s}_{n}(\lambda^{*})}{\sqrt{\frac{n'}{n + n'}(s(\lambda^{*}) - \gamma t\sqrt{s(\lambda^{*}) + \eta}) + \frac{n}{n + n'}(s(\lambda^{*}) - t\sqrt{s(\lambda^{*}) + \eta}) + \kappa^{-}}}\\
    & = \frac{s(\lambda^{*}) - \gamma t\sqrt{s(\lambda^{*}) + \eta} - \hat{s}_{n}(\lambda^{*})}{\sqrt{s(\lambda^{*}) - \frac{n + \gamma n'}{n + n'}\sqrt{s(\lambda^{*}) + \eta} + \kappa^{-}}}\\
    & \stackrel{\eqref{eq:step1_upper_sn_s}}{\ge} \frac{s(\lambda^{*}) - \gamma t\sqrt{s(\lambda^{*}) + \eta} - \hat{s}_{n}(\lambda^{*})}{\sqrt{s(\lambda^{*}) - \frac{n + \gamma n'}{n + n'}\sqrt{\kappa^{+}} + \kappa^{-}}}\\
    & = \frac{s(\lambda^{*}) - \gamma t\sqrt{s(\lambda^{*}) + \eta} - \hat{s}_{n}(\lambda^{*})}{\sqrt{s(\lambda^{*}) + \eta}}\\
    & \stackrel{\eqref{eq:step1_upper}}{\ge} (1 - \gamma)t.
  \end{align*}
  Similar to step 1 for \eqref{eq:slam_upper}, we complete the proof of \eqref{eq:step1_upper}.

  ~\\
  \noindent \textbf{Step 2 for \eqref{eq:slam_upper}:} using exactly the same proof of \eqref{eq:step2_upper}, we can prove that
  \begin{equation}
    \label{eq:step2_upper}
    \inf_{\lambda\in \R}\P\lb\frac{s(\lambda) - \hat{s}_{n'}(\lambda)}{\sqrt{s(\lambda) + \eta}}\le \gamma t\rb \ge 1 - \exp\{-g_1(t; n', \gamma, \eta)\}.
  \end{equation}
  
  ~\\
  \noindent \textbf{Step 3 for \eqref{eq:slam_upper}:} using exactly the same proof of \eqref{eq:step3_upper}, we can prove that
  \begin{equation}
    \label{eq:step3_upper}
    \P\lb \sup_{\lambda\in \R}\frac{\hat{s}_{n'}(\lambda) - \hat{s}_{n}(\lambda)}{\sqrt{\hat{s}_{n+n'}(\lambda) + \kappa^{-}}}\ge (1 - \gamma)t\rb\le \Delta(n + n')\exp\{g_2(t; n, n', \gamma, \kappa^{-})\}.
  \end{equation}

  ~\\
  \noindent \textbf{Step 4 for \eqref{eq:slam_upper}:}  putting \eqref{eq:step1_upper}, \eqref{eq:step2_upper} and \eqref{eq:step3_upper} together, we prove that for any $\gamma \in (0, 1)$ and $n' \in \Z^{+}$,
  \[\P\lb \sup_{\lambda\in \R}\frac{s(\lambda) - \hat{s}_{n}(\lambda)}{\sqrt{s(\lambda) + \eta}}\ge t\rb\le \frac{\Delta(n + n')\exp\{-g_2(t; n, n', \gamma, \kappa^{-})\}}{1 - \exp\{-g_1(t; n', \gamma, \eta)\}}.\]
  Since the right-handed side is deterministic, we can take infimum over $\gamma$ and $n'$, which yields \eqref{eq:slam_upper}.
\end{proof}

 \begin{proof}[\textbf{Proof of Theorem \ref{thm:normalized_vapnik_rademacher}}]
   We prove \eqref{eq:slam_upper_rademacher} first. We will use the same notation hereafter as in the proof of Theorem \ref{thm:normalized_vapnik}. By \eqref{eq:step1_upper} and \eqref{eq:step2_upper} with $n' = n$, it remains to prove that
   \begin{equation}
     \label{eq:goal_rademacher}
     \P\lb \sup_{\lambda\in \R}\frac{\hat{s}_{n}(\lambda) - \hat{s}'_{n}(\lambda)}{\sqrt{\hat{s}_{2n}(\lambda) + \eta}}\ge (1 - \gamma)t\rb\le \Delta(2n)\td{g}\lb \sqrt{\frac{n(1 + \eta)}{2}}(1 - \gamma)t\rb,
   \end{equation}
   where $\hat{s}'_{n} = \hat{s}_{n'}$ is the empirical CDF of $Z_{n+1}, \ldots, Z_{2n}$. Instead of conditioning on the unordered set of $\{Z_1, \ldots, Z_{2n}\}$, we condition on a more refined statistic
   \[\cZ_{\swap} \triangleq (\{Z_{1}, Z_{n+1}\}, \{Z_{2}, Z_{n + 2}\}, \ldots, \{Z_{n}, Z_{2n}\}).\]
   Note that $\cZ_{\swap}$ is a function of the unordered ste of $\{Z_{1}, \ldots, Z_{2n}\}$, the number of distinct $\cZ_{\swap}$ over $\lambda\in \R$ is at most $\Delta(2n)$. Thus, by a union bound,
   \[\P\lb \sup_{\lambda\in \R}\frac{\hat{s}_{n}(\lambda) - \hat{s}'_{n}(\lambda)}{\sqrt{\hat{s}_{2n}(\lambda) + \eta}}\ge (1 - \gamma)t\rb\le \Delta(2n)\sup_{\lambda\in \R}\P\lb \frac{\hat{s}_{n}(\lambda) - \hat{s}'_{n}(\lambda)}{\sqrt{\hat{s}_{2n}(\lambda) + \eta}}\ge (1 - \gamma)t\rb.\]
   Conditional on $\cZ_{\swap}$, $\hat{s}_{2n}$ is deterministic and for any event $A$,
   \[\lb S(\lambda; Z_{i}) - S(\lambda; Z_{n + i})\rb_{i=1}^{n}\stackrel{d}{=} \lb S(\lambda; Z_{n+i}) - S(\lambda; Z_{i})\rb_{i=1}^{n}\]
   where $\eps_{1}, \ldots, \eps_{n}$ are i.i.d. Rademacher random variables, i.e. $\P(\eps_{i} = \pm 1) = 1/2$, and are independent of $\cZ_{\swap}$. As a result,
   \begin{align}
     \lefteqn{\P\lb \frac{\hat{s}_{n}(\lambda) - \hat{s}'_{n}(\lambda)}{\sqrt{\hat{s}_{2n}(\lambda) + \eta}}\ge (1 - \gamma)t \mid \cZ_{\swap}\rb}\nonumber\\
     & = \P\lb \frac{\sum_{i=1}^{n}(S(\lambda; Z_{i}) - S(\lambda; Z_{n+i}))\eps_{i}}{\sqrt{\sum_{i=1}^{n}(S(\lambda; Z_{i}) + S(\lambda; Z_{n+i})) + 2n\eta}} \ge \sqrt{\frac{n}{2}}(1 - \gamma)t\mid \cZ_{\swap}\rb\nonumber\\
     & = \P\lb \frac{\sum_{i=1}^{n}\sqrt{1 + \eta}(S(\lambda; Z_{i}) - S(\lambda; Z_{n+i}))\eps_{i}}{\sqrt{\sum_{i=1}^{n}(S(\lambda; Z_{i}) + S(\lambda; Z_{n+i})) + 2n\eta}} \ge \sqrt{\frac{n(1 + \eta)}{2}}(1 - \gamma)t\mid \cZ_{\swap}\rb.\label{eq:rademacher1}
   \end{align}
   Let
   \[a_{i} = \frac{\sqrt{1 + \eta}(S(\lambda; Z_{i}) - S(\lambda; Z_{n+i}))}{\sqrt{\sum_{i=1}^{n}(S(\lambda; Z_{i}) + S(\lambda; Z_{n+i})) + 2n\eta}}.\]
   Then $a_{i}$ is deterministic conditional on $\cZ_{\swap}$ and
   \begin{align*}
     \sum_{i=1}^{n}a_{i}^{2}
     &= \frac{(1 + \eta)\sum_{i=1}^{n}(S(\lambda; Z_{i}) - S(\lambda; Z_{n+i}))^2}{\sum_{i=1}^{n}(S(\lambda; Z_{i}) + S(\lambda; Z_{n+i})) + 2n\eta}\\
     &\stackrel{(i)}{\le} \frac{(1 + \eta)\sum_{i=1}^{n}(S(\lambda; Z_{i}) + S(\lambda; Z_{n+i}))}{\sum_{i=1}^{n}(S(\lambda; Z_{i}) + S(\lambda; Z_{n+i})) + 2n\eta}\\
     &\stackrel{(ii)}{\le} 1,
   \end{align*}
   where (i) uses the fact that $(S(\lambda; Z_{i}) - S(\lambda; Z_{n+i}))^2 \le S(\lambda; Z_{i})^2 + S(\lambda; Z_{n+i})^2 = S(\lambda; Z_{i}) + S(\lambda; Z_{n+i})$, and (ii) uses the fact that $\sum_{i=1}^{n}(S(\lambda; Z_{i}) + S(\lambda; Z_{n+i}))\le 2n$. Then \eqref{eq:rademacher1} implies that
   \[\P\lb \frac{\hat{s}_{n}(\lambda) - \hat{s}'_{n}(\lambda)}{\sqrt{\hat{s}_{2n}(\lambda) + \eta}}\ge (1 - \gamma)t \mid \cZ_{\swap}\rb\le \sup_{\|a\|_{2}^{2} \le 1}\P\lb \sum_{i=1}^{n}a_{i}\eps_{i}\ge \sqrt{\frac{n(1 + \eta)}{2}}(1 - \gamma)t\rb.\]
   By the Bentkus-Dzindzalieta inequality (Proposition \ref{prop:bentkus2015tight}),
   \[\P\lb \sum_{i=1}^{n}a_{i}\eps_{i}\ge \sqrt{\frac{n(1 + \eta)}{2}}(1 - \gamma)t\rb\le \td{g}_{1}\lb \sqrt{\frac{n(1 + \eta)}{2}}(1 - \gamma)t\rb.\]
   By the Pinelis inequality (Proposition \ref{prop:pinelis2012asymptotically}),
   \[\P\lb \sum_{i=1}^{n}a_{i}\eps_{i}\ge \sqrt{\frac{n(1 + \eta)}{2}}(1 - \gamma)t\rb\le \td{g}_{2}\lb \sqrt{\frac{n(1 + \eta)}{2}}(1 - \gamma)t\rb.\]
   Finally, since $\eps_{i}$ is subgaussian with parameter $1$ and $\|a\|_{2}^{2} = 1$, $\sum_{i=1}^{n}a_{i}\eps_{i}$ is also subgaussian with parameter $1$. By Hoeffding's inequality, we have
   \[\P\lb \sum_{i=1}^{n}a_{i}\eps_{i}\ge \sqrt{\frac{n(1 + \eta)}{2}}(1 - \gamma)t\rb\le \td{g}_{3}\lb \sqrt{\frac{n(1 + \eta)}{2}}(1 - \gamma)t\rb.\]
   Putting the pieces together, \eqref{eq:slam_upper_rademacher} is proved.

   ~\\
   \noindent Similarly, to prove \eqref{eq:slam_upper_rademacher}, it remains to prove
   \begin{equation*}
     \P\lb \sup_{\lambda\in \R}\frac{\hat{s}'_{n}(\lambda) - \hat{s}_{n}(\lambda)}{\sqrt{\hat{s}_{2n}(\lambda) + \eta}}\ge (1 - \gamma)t\rb\le \Delta(2n)\td{g}\lb \sqrt{\frac{n(1 + \eta)}{2}}(1 - \gamma)t\rb.
   \end{equation*}
   Since $\hat{s}_{n}$ and $\hat{s}'_{n}$ are symmetric, it is implied by \eqref{eq:goal_rademacher}. Thus, the proof of \eqref{eq:slam_upper_rademacher} is also completed.
 \end{proof}

\end{document}

%% file: figures/tables/0_2-coco-table.tex
\footnotesize
\begin{tabular}{lccccccc} 
\toprule
Method & 50\% & 75\% & 90\% & 99\% & 99.9\% \\ 
\midrule
Uniform & 3 & 4 & 6 & 11 & 13\\
Bonferroni & 3 & 4 & 7 & 11 & 14\\
Fixed Sequence
(Multi-Start) & 3 & 4 & 7 & 11 & 14\\
Fixed Sequence & 3 & 5 & 7 & 12 & 14\\
\bottomrule
\end{tabular}

%% file: figures/tables/ar-process-table.tex
\footnotesize
\begin{tabular}{lccc} 
\toprule
Method & $\alpha$=0.1 & $\alpha$=0.15 & $\alpha$=0.2 \\
\midrule
Empirical risk $<\alpha$ & 0.61& 0.67& 0.73\\
Fixed Sequence & $\emptyset$ & 0.53& 0.56\\
Bonferroni & $\emptyset$& $\emptyset$& 0.56\\
Uniform & $\emptyset$& $\emptyset$& $\emptyset$\\
\bottomrule
\end{tabular}

%% file: figures/tables/0_5-coco-table.tex
\footnotesize
\begin{tabular}{lccccccc} 
\toprule
Method & 50\% & 75\% & 90\% & 99\% & 99.9\% \\ 
\midrule
Uniform & 5 & 9 & 12 & 20 & 24\\
Bonferroni & 5 & 9 & 13 & 21 & 24\\
Fixed Sequence 
(Multi-Start) & 6 & 9 & 13 & 21 & 25\\
Fixed Sequence & 6 & 9 & 13 & 21 & 26\\
\bottomrule
\end{tabular}

%% file: algorithms/bardenet.tex
\begin{algorithm}
    \caption{FDR Calibration via Uniform Concentration}
    \label{algo:bardenet}
    \begin{algorithmic}[1]
        \Require Nested-set-valued function $\T_{\lambda}$, desired FDR $\alpha$, calibration set $(X_1,Y_1),...,(X_n,Y_n)$, step size $\zeta$.
        \Procedure{LargestSet}{$\T_{\lambda}, \alpha, (X_1,Y_1),...,(X_n,Y_n)$}
        \State $\lambda \gets 1$
        \While{$\lambda > 0 \And  \alpha \ge  \widehat{\text{FDR}}\big(\T_{{\lambda}}\big) + t(\eta^*; \delta) \sqrt{\widehat{\text{FDR}}\big(\T_{{\lambda}}\big) + \eta^* + t(\eta^{*}; \delta)^2 / 4} + t(\eta^{*}; \delta) / 2$}
            \vspace{1mm}
            \State $\lambda \gets \lambda - \zeta$
        \EndWhile
        \State \textbf{return} $\lambda$
        \EndProcedure
        \Ensure A parameter $\hat{\lambda}$ that controls the FDR at level $\alpha$ with probability $\delta$.
    \end{algorithmic}
\end{algorithm}